\documentclass[twoside,11pt]{article}

\usepackage{wustyle-jmlr}

\usepackage{subcaption}
\usepackage{jmlr2e}

\newtheorem{assumption}[theorem]{Assumption}

\newcommand{\cF}{\mathcal{F}}
\newcommand{\cN}{\mathcal{N}}
\newcommand{\RR}{\mathbb{R}}
\newcommand{\PP}{\mathbb{P}}
\newcommand{\EE}{\mathbb{E}}

\newcommand{\bR}{\mathbb{R}}
\newcommand{\bF}{\mathbf{F}}
\newcommand{\cG}{\mathcal{G}}
\newcommand{\cY}{\mathcal{Y}}
\newcommand{\cQ}{\mathcal{Q}}
\newcommand{\bS}{\mathbb{S}}
\renewcommand{\SS}{\mathbb{S}}
\newcommand{\rmd}{\mathrm{d}}

\newcommand{\ia}{\frac{1}{r^2}}
\newcommand{\transpose}{^{\operatorname{T}}}
\newcommand{\step}{\operatorname{step}}
\newcommand{\relu}{\operatorname{ReLU}}

\newcommand{\qed}{\hfill\blacksquare}

\jmlrheading{1}{2000}{1-48}{4/00}{10/00}{meila00a}{}

\usepackage{lastpage}
\jmlrheading{23}{2022}{1-\pageref{LastPage}}{8/21; Revised 1/22}{2/22}{21-0927}{Lei Wu and Jihao Long}
\ShortHeadings{The separation between neural networks and linear methods}{Wu and Long}

\firstpageno{1}

\begin{document}

% \title{Kolmogorov widths of shallow neural networks: A tight estimate via spectral decay}
% \title{The separation of neural networks and linear methods: A comprehensive analysis via spectral decay}
\title{A spectral-based analysis of the separation between two-layer neural networks and linear methods}

\author{\name Lei Wu \email leiwu@math.pku.edu.cn\\
       \addr School of Mathematical Sciences, Peking University
       \AND
		\name Jihao Long \email jihaol@princeton.edu \\
       \addr PACM, Princeton University 
}

\editor{Joan Bruna}
\maketitle

\begin{abstract}%
We propose a spectral-based approach to analyze how two-layer neural networks separate from linear methods in terms of approximating high-dimensional functions. We show that quantifying this separation can be reduced to estimating the Kolmogorov width of two-layer neural networks, and the latter can be further 
characterized by using the spectrum of an associated kernel. Different from previous work, our approach allows obtaining upper bounds, lower bounds, and identifying explicit hard functions in a united manner. We provide a systematic study of how the choice of activation functions affects the separation, in particular the  dependence on the input dimension. Specifically, for nonsmooth activation functions, we extend known results to more activation functions with sharper bounds. As concrete examples, we prove that any single neuron can  instantiate the separation between neural networks and random feature models. For smooth activation functions, one surprising finding is that the separation is negligible unless the norms of inner-layer weights are polynomially large  with respect to the input dimension. By contrast, the separation for nonsmooth activation functions is independent of the norms of inner-layer weights.
\end{abstract}

\begin{keywords}
Kolmogorov width, two-layer neural network, single neuron, spectral decay, curse of dimensionality
\end{keywords}

% \tableofcontents

\section{Introduction}
Neural network-based machine learning has achieved remarkable performances in many applications, such as computer vision, natural language processing,  and scientific computing. One of the key reasons behind these successes is that neural networks can efficiently approximate certain high-dimensional functions \citep{barron1993universal,barron1994approximation,bach2017breaking,ongie2019function,ma2019priori,poggio2018theory,weinan2021barron,gribonval2021approximation,suzuki2018adaptivity,klusowski2016risk}, while traditional linear methods cannot \citep{barron1993universal,barron1992neural,kurkova2002comparison,siegel2021sharp}. In this paper, we provide a comprehensive study of this phenomenon for two-layer neural networks.

% In this paper, we provide 

% A lot of effort has been made to establish mathematical foundations behind this observation, such as shallow neural networks \citep{barron1993universal,barron1994approximation,bach2017breaking,ongie2019function,ma2019priori,siegel2020approximation}, deep neural networks \citep{poggio2018theory,daniely2017depth,weinan2021barron,eldan2016power,gribonval2021approximation,suzuki2018adaptivity,klusowski2016risk}, just to name a few. 

% To understand the superiority of neural networks, 

Consider two-layer neural networks  given by
\begin{equation}\label{eqn: 2lnn}
    f_m(x;\theta) = \sum_{j=1}^m a_j \sigma(w_j^Tx+b_j),
\end{equation}
where $w_j\in\RR^d,a_j\in\RR, b_j\in\RR$ and $\sigma: \RR\mapsto\RR$ is the nonlinear activation function. We use $\theta=\{a_j,w_j,b_j\}_{j=1}^m$ to denote all the parameters and $m$ to denote the network width. Denote by $\cN_m$ the set of two-layer neural networks of width $m$:
\begin{equation}\label{eqn: def-hypothesis}
\cN_m =\Big\{f_m(\cdot;\theta)\,\big|\,\sum_{j=1}^m |a_j|\leq 1, \max_{j}(\|w_j\|_2+|b_j|)\leq r\Big\}.
\end{equation}

To study the separation between neural networks and linear methods, we consider target functions in  
\begin{equation}\label{eqn: barron-space}
\cN=\overline{\cup_{m=1}^\infty \cN_m},
\end{equation}
 which is  the closure (with respect to the $L^\infty$ metric) of all finite-width neural networks with bounded norms defined above.  Note that the function class $\cN$ depends on the activation function $\sigma$ and the size of inner-layer weights $r$, which are omitted in the notation  for simplicity. The function class $\cN$ has been widely used to analyze properties of two-layer neural networks including but not limited to the separation from linear methods. For more details, we refer to \cite{barron1993universal,barron1992neural,kurkova2002comparison,bach2017breaking,ma2019priori,weinan2021barron} and the references therein. In this paper, we will focus on studying the separation in terms approximating functions in $\cN$.

Denote by $\rho$ the input distribution and the approximation error will be measured with respect to $\rho$.
Mathematically speaking, quantifying the separation boils down to estimating the following two quantities:
\begin{enumerate}
\item[(1)] The approximability by two-layer neural networks:
\begin{equation}
e_m(\cN) := \sup_{f\in \cN}\inf_{g\in \cN_m}\|f-g\|_{L^2(\rho)}^2,
\end{equation}
which describes the worst-case error of approximating functions in $\cN$ with two-layer neural networks.
\item[(2)] The (in)approximability by linear methods: 
\begin{equation}
 \omega_m(\cN) :=  \inf_{\phi_1,\cdots,\phi_m\in L^2(\rho)} \sup_{f\in\cN} \inf_{c_1,\dots,c_m\in\RR} \|f-\sum_{j=1}^m c_j \phi_j\|_{L^2(\rho)}^2,
\end{equation}
which describes the worst-case error of approximating functions in $\cN$ with the best-possible fixed features. Note that $\omega_m(\cN)$ is exactly the \emph{Kolmogorov width} of $\cN$ \citep{kolmogoroff1936uber,lorentz1966approximation}.
\end{enumerate}
When $\omega_m(\cN)$ is significantly larger than $e_m(\cN)$, we say that there is a \emph{separation} between two-layer neural networks and linear methods (with the best-possible fixed features). In other words, there exist functions in $\cN$ such that they can be efficiently approximated by neural networks but not by linear methods.

We are particularly interested in how the separation depends on the input dimension $d$. Does the separation become more significant in higher dimensions?  To facilitate the statement, we use the concept of curse of dimensionality (CoD) \citep{bellman1957}. A quantity $q_m$ is said to be free of the CoD if $q_m\leq \poly(d)m^{-\beta}$ with $\beta$ being a constant independent of $d$. On the contrary, we say $q_m$ exhibits or suffers from the CoD if $q_m\geq \poly(1/d)$ for all $m\leq \exp(d)$. A specific rate that achieves the CoD is $q_m\geq \poly(1/d)m^{-\beta/d}$ with  $\beta$ being a constant independent of $d$. In particular, when the Kolmogorov width $\omega_m(\cN)$ exhibits the CoD,  we must need exponentially many features for apporoximating functions in $\cN$.

The seminal work \citep{barron1993universal} identified a subclass $\cN_{\text{F}}\subset \cN$, determined by the first-order moment of the Fourier transform, and  showed $e_m(\cN_{\text{F}})=O(1/m)$. This bound means that approximating functions in $\cN_{\text{F}}$ with neural networks is free of the CoD; in other words, the approximation rate is dimension-independent. Later, \cite{barron1992neural} extended this result to the whole class $\cN$. 
Along this line of work, \citep{breiman1993hinging,makovoz1998uniform,kuurkova1997estimates,kurkova2001bounds,kurkova2002comparison,bach2017breaking,ma2019priori,siegel2020approximation} improved and extended the upper bound, showing $e_m(\cN) = O(1/m)$ holds for very general activation functions.  
Given this upper bound of $e_m(\cN)$,  what remains is estimating the Kolmogorov width $\omega_m(\cN)$. 

 \cite{barron1993universal,barron1992neural} proved that $\omega_m(\cN)\geq C d^{-1} m^{-1/d}$, which implies that approximating functions $\cN$ with linear methods (even using the best-possible features) suffers from the CoD. Together with $e_m(\cN)=O(m^{-1})$,  \cite{barron1993universal,barron1992neural} established the first CoD-type separation for neural networks and linear methods.    However, the lower bound is limited to sigmoidal activation functions, such as the sigmoid and Heaviside step function. 
\cite{barron1993universal} obtained the lower bound by explicitly constructing exponentially many orthogonal functions in $\cN$ via the Fourier transform. \cite{kurkova2002comparison} provided a systematic study of this orthogonal function argument for general function classes. But the application to neural networks is still limited to the case considered in   \cite{barron1993universal}, since the Fourier-based  construction of orthogonal functions in \cite{barron1993universal} is specific to sigmoidal activations. It is unclear how to extend it to general cases. 
On the other hand, these works did not consider the upper bound of $\omega_m(\cN)$, which is critical for understanding when and how the separation from linear methods becomes negligible. For instance, if  $\omega_m(\cN)=O(m^{-1})$, then there is no separation between two-layer neural networks and linear methods for approximating functions in $\cN$ since $e_m(\cN)=O(m^{-1})$.

In addition, the separation results mentioned above rely on worst-case analyses and only guarantee the existence of functions that are hard to approximate with linear methods.
Another important but less-explored question is: Can we identify some specific hard functions that instantiate the separation? A recent progress was made by \cite{ma2019priori} and \cite{yehudai2019power} for the separation from the specific random feature model (RFM)~\citep{rahimi2007random}: $h_m(x;a)=\sum_{j=1}^m a_j\sigma(w_j^Tx)$, where  $\{w_j\}$ are independently and uniformly drawn from the unit sphere. This model may look similar to two-layer neural networks. The difference is that the $\{w_j\}$ in two-layer neural networks are adaptive to the target function, whereas $\{w_j\}$ in the RFM are fixed.

Specifically, \cite{ma2019priori} and \cite{yehudai2019power} showed that a single neuron: $x\mapsto \sigma(v^Tx+b)$ is enough to separate the two methods. \cite{ma2019priori} numerically showed that approximating a single neuron with random features suffers from the CoD;
\cite{yehudai2019power} provided a partial theoretical explanation. \cite{yehudai2019power}  proved that there exist constants $C_0, C_1 > 0$ such that,
if $m \leq e^{C_1d}$ and the coefficients satisfy $\max_{j\in [m]} |c_j|\leq e^{C_1d}/m$, then there exists a $b^*\in\RR$ such that for any $v\in\RR^d$ with $\|v\|=d^3$, the error of approximating $x\mapsto \sigma(v^Tx+b^*)$  with the random features is larger than $C_0$. However, this theoretical result is quite unsatisfying in two aspects: (1) The bias term $b^*$ still needs to be chosen in an adversary way, which means that this function is not completely explicit; (2) The coefficient magnitudes need to be bounded. Consequently, we cannot fully claim the separation for approximating this specific function. 
In addition,  the analysis there is essentially also based on explicit constructions of  nearly orthogonal functions \citep{malach2020hardness}.

\vspace*{-.2em}
\subsection{Overview of our contributions}
\vspace*{-.1em}

In this paper, we present a spectral-based approach to analyze the separation, which reduces the problem to estimating the eigenvalues of an associated kernel.  Specifially, consider a general parametric feature $\varphi:X\times \Omega \mapsto \RR$ and let $\Phi=\{\varphi(\cdot;v)\,|\, v \in \Omega\}$ be the class of features. The key quantity in our analysis is the following average version of Kolmogorov width: 
\begin{equation}
A_m^{(\pi)}(\Phi) = \inf_{\phi_1,\dots,\phi_m \in L^2(\rho)}\EE_{v \sim \pi}\inf_{c_1,\dots,c_m \in \bR}\|\varphi(\cdot;v) - \sum_{j=1}^mc_j\phi_j\|_{L^2(\rho)}^2,
\end{equation}
where $\pi$ is a probability distribution on $\Omega$.
We prove in Lemma \ref{pro: gen-feature} that 
$
A_m^{(\pi)}(\Phi)= \sum_{j=m+1}^\infty \lambda_j,
$
where $\{\lambda_j\}_{j\geq 1}$ are the eigenvalues in a non-increasing order of the kernel: $k_\pi(x,x') = \langle \varphi(x;\cdot), \varphi(x';\cdot)\rangle_\pi$. Then we apply this general result to the two-layer neural network case where $\varphi(x;v)=\sigma(w^Tx+b)$. The specfic procedures go as follows.
\begin{itemize}
\item
Let 
$
    \cQ^{\gamma,b} = \{\sigma_v(x) = \sigma(\gamma v^T x +b)| v \in \mathbb{S}^{d-1}\}
$
be the class of single neuron and $\tau_{d-1}$ be the uniform distribution over $\SS^{d-1}$.
We show in Theorem \ref{proposition: 2lnn-complete-characterization} that when $\rho=\tau_{d-1}$, there exists a $C_d=\poly(d)$ such that 
\begin{equation}\label{eqn: Kw-bound}
\sup_{\gamma + |b|\le r}A_m^{(\tau_{d-1})}(\cQ^{\gamma,b})\le \omega_m(\cN)\le C_d \sup_{\gamma + |b|\le r}A_m^{(\tau_{d-1})}(\cQ^{\gamma,b}).
\end{equation}

\item Moreover, by Lemma \ref{pro: gen-feature}, $A_m^{(\tau_{d-1})}(\cQ^{\gamma,b})=\sum_{j=m+1}^\infty \lambda_j$,   where $\{\lambda_j\}_{j\geq 1}$ denote the eigenvalues corresponding to the kernel: 
 $k^{(\gamma,b)}(x,x')=\EE_{v\sim\tau_{d-1}}[\sigma(\gamma v^Tx+b)\sigma(\gamma v^Tx'+b)]$.  
\end{itemize}
We thus provide a tight characterization of $\omega_m(\cN)$ using the spectrum of $k^{(\gamma,b)}$. For ReLU$^\alpha$  activations, we even have $C_d=O(1)$, where the characterization becomes exact. Here the ReLU$^\alpha$ activation function is defined by $\sigma(z)=\max(0,z^\alpha)$. Note that the kernel $k^{(\gamma,b)}$ is in a dot-product form, for which we can apply the harmonic analysis to obtain explicit estimates of the eigenvalues  \citep{smola2001regularization,bach2017breaking}.

 The above approach does not need the explicit construction of orthogonal functions, thereby working
 for very general activation functions, including the commonly-used ReLU, Gaussian error linear unit (GELU) \citep{hendrycks2016gaussian}, Swish/SiLU \citep{ramachandran2017searching,elfwing2018sigmoid}.  Moreover, it allows obtaining upper bounds, lower bounds, and explicit hard functions simultaneously.  Using this approach, we provide a comprehensive study of how the smoothness of $\sigma$ affects the decay of $\omega_m(\cN)$, in particular the dependence on the input dimension. Our specific findings are summarized as follows.

\begin{itemize}
\item For the nonsmooth ReLU$^\alpha$ activation functions,  we show in Proposition \ref{sec: nonsmooth-single-neuron} that $\omega_m(\cN)\geq C(1/d,\alpha)m^{-\frac{2\alpha+1}{d-1}}$ with $C(1/d,\alpha)$ depending on $1/d$ polynomially. Combining with the known bounds $e_m(\cN)=O(1/m)$,  we establish the first CoD-type separation for general ReLU$^\alpha$ activations, which includes the result of \cite{barron1993universal} as a special case.  In Theorem \ref{thm: gen-1}, we  prove that the preceding lower bound holds for any single neuron, which removes all the restrictions of \cite{yehudai2019power}, including the boundedness of coefficient magnitudes and the adversarial choice of bias term. Thus we can now truly claim the separation from the RFM for this specific function.  Moreover,  our results hold for any ReLU$^\alpha$ activation, whereas
\cite{yehudai2019power} only considers the ReLU case, i.e., $\alpha=1$.

\item 
For smooth activation functions, we show that whether $\omega_m(\cN)$  exhibits the CoD or not depends on the norms of inner-layer weights, i.e., the value of $r$ in Eq.~\eqref{eqn: def-hypothesis}. Specifically, Theorem \ref{thm: 2lnn-smooth-ub} shows that when $r=1$, $\omega_m(\cN)\leq d/m$ for activation functions that satisfy certain derivative boundedness condition.  Proposition \ref{pro: 2lnn-smooth-ub-2} further provides a fine-grained characterization of the dependence on $r$ for the specific arctangent activation, showing $\omega_m(\cN)\leq d^4r^2m^{-\max(1/2, 1/r^2)}$. These  rates are dimension-independent and moreover, they can be  achieved by using spherical harmonics, i.e., the homogeneous harmonic polynomials constrained on $\SS^{d-1}$. Therefore, neural networks in these cases do not perform (significantly) better than polynomials.
On the contrary, Theorems  \ref{thm: 2lnn-arctan-lower-bound} and \ref{thm: low-bound-actan} show that the CoD-type separation can be recovered for sigmoid-like and ReLU-like smooth activations when $r\geq d^C$ for some positive constant $C$. In particular,  $C>1/2$ is enough for the specific arctangent activation function.
\end{itemize}

We notice that there is a concurrent work \citep{siegel2021sharp}, obtaining the same decay rates of  $\omega_m(\cN)$ for the ReLU$^\alpha$ activations. However, the bounds there are not very useful in the high-dimensional regime since \cite{siegel2021sharp} did not show their constants depending on $d$ polynomially. On the other hand, the technique used in \cite{siegel2021sharp} is still based on the orthogonal function argument developed in \cite{barron1993universal,kurkova2002comparison}, whereas ours is completely different. 

Essentially, our spectral-based approach, in particular the bound \eqref{eqn: Kw-bound}, provides a tight characterization of the Kolmogorov width for two-layer neural networks, which is even almost exact for the ReLU$^\alpha$ activation function.
It is also applicable to analyze other properties that are related to Kolmogorov width.  For instance, \cite{siegel2021sharp} obtained a tight lower bound of the metric entropy (logarithm of the covering number) \citep{kolmogorov1958linear} of two-layer neural networks by using our results.

\subsection{Other related work}

\paragraph*{Approximation with random features}
Our work is also related to \cite{ghorbani2021linearized}, which shows that the random feature model effectively fits polynomials.
Let $\gamma=\cN(0,1)$ be the standard normal distribution. Specifically, for the single neuron target function, \cite{ghorbani2021linearized} showed that the approximating error with $N$ random features is roughly controlled by  $\|\sigma_{>\ell}\|^2_{L^2(\gamma)}$ when $d^{\ell+\delta}\leq N\leq d^{\ell+1-\delta}$ for small $\delta>0$. Here, $\sigma_{>\ell}(\cdot)$ is the projection of $\sigma$ orthogonally to the subspace spanned by polynomials of minimum degree $\ell$. However, no explicit estimate is provided, and it is, therefore, unclear how the approximability depends on the input dimension. Moreover, the analysis of \cite{ghorbani2021linearized} is limited to the specific forms of features. By contrast, our analysis is applicable to very general features.

\paragraph*{Learning neural networks activated by smooth functions}
\cite{livni2014computational} proved that two-layer neural networks activated by the sigmoid function can be efficiently approximated by polynomials, if the norms of inner-layer weights are bounded by a constant independent of $d$. Then, they use this approximation result to show that these networks can be learned efficiently in polynomial time (see also \cite{zhang2016l1}).  In this paper, we extend this result to more general smooth activations, which include all the commonly-used ones, such as tanh, softplus, GELU. 
Moreover, we prove that when the  norms of inner-layer weights are large than $d^\beta$ with $\beta$ being a small positive constant, there does not exist fixed features such that the linear method can avoid the CoD.
These improvements over \cite{livni2014computational} benefit from  the spectral-based approach developed in our study, which is quite different from the techniques used in \cite{livni2014computational}. 

\paragraph*{The spectrum  of the associated kernel}
It is well-known that the spectrum of the kernel $k_\pi$ plays an important role in analyzing  the corresponding RFM \citep{carratino2018learning,bach2017equivalence,bach2017breaking,mei2019generalization}. Recently, it is also extensively explored to understand neural networks in the kernel regime \citep{daniely2017sgd,xie2017diverse,jacot2018neural,chen2021deep,bietti2019inductive,bietti2021deep,scetbon2021spectral,bach2017breaking}. In contrast to these works, we reveal that the spectrum of that kernel also play a fundamental role in  separating  neural networks from linear methods, which is essentially a property beyond the kernel regime.

At a technical point, our work also bears similarity with \cite{smola2001regularization} and \cite{bach2017breaking} since we all rely on the harmonic analysis of functions on $\SS^{d-1}$. Specifically, \cite{smola2001regularization} provided the integral representation of eigenvalues for general dot-product kernels; \cite{bach2017breaking} provided detailed calculations for the specific kernel $k_\pi$. In particular,  \cite{bach2017breaking} obtained (complicated) analytic expressions of the eigenvalues for ReLU$^\alpha$ activations, which are expressed in terms of Gamma functions. 
The reduced problem in our analysis is also estimating the eigenvalues of the kernel as explained above. We adopt the integral representation in \cite{smola2001regularization,bach2017breaking} and the analytic expressions in \cite{bach2017breaking}.  The specific differences from \cite{bach2017breaking} are listed below. 
\begin{itemize}
\item For ReLU$^\alpha$ activation functions, we obtain non-asymptotic estimates of the eigenvalues with the constants depending on $d$ polynomially, whereas \cite{bach2017breaking} only considered the asymptotic regime without tracking the constants. This improvement is crucial for understanding the  separation  in high dimensions. 
\item We provide  estimates of the eigenvalues for smooth activation functions, which are not covered in \cite{bach2017breaking}. Moreover, for the specific arctangent activation, we further provide a fine-grained characterization by showing that the eigenvalue can  be expressed analytically using Gaussian hypergeometric functions \citep{olver2010nist}. %The estimate of this part is highly non-trivial and we use for example the Parseval's theorem, properties of Bessel function of the first kind, integral representations of Gaussian hypergeometric function, etc. {\color{red} I think it's better to stop at 'The esimate of this part is highly non-trival' and omit the details.}
\item  \cite{bach2017breaking} used the spectrum of that kernel to compute the reproducing kernel Hilbert space (RKHS) \citep{aronszajn1950theory} norm induced by that RFM. However, our interest is analyzing how two-layer neural networks outperform linear methods.  
\end{itemize}

% \section{Problem setup and main results}
\vspace*{-1em}
\section{Preliminaries}
\paragraph*{Notation.} Let $\SS^{d-1}=\{x\in\RR^d : \|x\|_2=1\}$, $\omega_{d-1}=\frac{2\pi^{d/2}}{\Gamma(d/2)}$ be the surface area of $\SS^{d-1}$, and $\tau_{d-1}$ be the uniform distribution over $\SS^{d-1}$. For any $\Omega$, denote by $\cP(\Omega)$ the set of probability measures over $\Omega$. For a probability measure $\gamma$, for simplicity we use $\langle\cdot,\cdot\rangle_\gamma$ and $\|\cdot\|_{\gamma}$ to denote the $L^2(\gamma)$ inner product and norm, respectively. Given a subset $B$, denote by $1_{B}$ the indicator function of $B$, which has value $1$ at points of $B$ and $0$ otherwise. 
Let $\Gamma(x)$ be the Gamma function 
and $\binom{n}{m}=\frac{\Gamma(n+1)}{\Gamma(m+1)\Gamma(n-m+1)}$ be the binomial coefficient. We shall use $\poly(\beta)$ to denote a quantity that depends on  $\beta$ polynomially.
We use $X\sim Y$, if there exist absolute constants $C_1,C_2>0$ such that $C_1 Y\leq X \leq C_2 Y$. $X\lesssim Y$ means $X\leq CY$ for an absolute  constant $C>0$, and $X\gtrsim Y$ is defined analogously.

\subsection{Legendre polynomials and spherical harmonics}
We shall focus on the case of $\rho=\tau_{d-1}$, where the eigenvalues of the kernel of interest can be explicitly estimated. Hence, we need to prepare some basic techniques for analyzing functions on $\SS^{d-1}$.
Denote by $P_k$ the associated Legendre polynomials of degree $k$ in $d$ dimensions \citep[Section~2.6]{atkinson2012spherical}, which satisfies the following recursive formula   \cite[Equation~2.86]{atkinson2012spherical}:
\begin{equation}\label{eq: recursive}
\begin{aligned}
	P_0(t) &= 0,\, P_1(t) = t, \\
    P_k(t) &= \frac{2k +d -4}{k+d-3}t P_{k-1}(t) - \frac{k-1}{k+d-3}P_{k-2}(t),\, k \ge 2.
\end{aligned}
\end{equation}
Note that $\{P_k\}_{k=0}^\infty$ are the orthogonal polynomials with respect to the distribution $p_d(t)=(1-t^2)^{\frac{d-3}{2}}/B(\frac 1 2,\frac{d-1}{2})$, i.e., the distribution of $x_1$ for $x\sim\tau_{d-1}$. Specifically,
\begin{equation}
\int_{-1}^{1} P_{k}(t) P_{j}(t)\left(1-t^{2}\right)^{(d-3) / 2} \dd t=\delta_{j k} \frac{\omega_{d-1}}{\omega_{d-2}} \frac{1}{N(d, k)},
\end{equation}
where
\begin{equation*}
    N(d,k)=\frac{2k+d-2}{k}\binom{k+d-3}{d-2}
\end{equation*}
\cite[Equation~2.67 and 2.68]{atkinson2012spherical}.
Notice that the above equation can be rewritten as follows:
\begin{equation*}
    \int_{\bS^{d-1}}P_k(x\transpose y) P_j(x\transpose y)\rmd \tau_{d-1}(y) = \frac{\delta_{jk}}{N(d,k)}, \,\forall x \in \bS^{d-1}.
\end{equation*}
The Rodrigues's formula  gives a closed-form expression of $P_k$:  
\begin{align}\label{eqn: Rodrigues}
P_{k}(t)=\left(-\frac{1}{2}\right)^{k} \frac{\Gamma((d-1)/ 2)}{\Gamma(k+(d-1)/ 2)}\left(1-t^{2}\right)^{(3-d) / 2}\left(\frac{d}{d t}\right)^{k}\left(1-t^{2}\right)^{k+(d-3) / 2}.
\end{align}
The polynomial $P_k$ is even (resp. odd) when $k$ is even (resp. odd). 
%where $N(d,k)=\frac{2k+d-2}{k}\binom{k+d-3}{d-2}$ denotes the multiplicity of spherical harmonics of degree $k$ on $\SS^{d-1}$.

Let $\cY_k^d$ be the space of all homogeneous harmonic polynomials of degree $k$ in $d$ dimensions restricted on $\SS^{d-1}$; the dimension of the space $\cY_k^d$ is $N(d,k)$. Let $\{Y_{k,j}\}_{1 \le j \le N(d,k)}$ be an orthonormal basis of $\cY_{k}^d$ in $L^2(\tau_{d-1})$. Then $Y_{k,j}:\SS^{d-1}\mapsto\RR$ denotes the $j$-th spherical harmonics of degree $k$. Then $\{Y_{k,j}\}_{k \in \mathbb{N}, 1 \le j \le N(d,k)}$  forms an orthonormal basis of $L^2(\tau_{d-1})$  \cite[Section~2.1.3, Corollary~2.15~and Theorem~2.38]{atkinson2012spherical}. 
% When the dimension $d$ is clear from the context, we will omit the superscript $d$ to ease the notation. 

The spherical harmonics is related to the Legendre polynomials: 
\begin{equation}\label{eqn: sum_of_Y}
    \sum_{j=1}^{N(d,k)} Y_{k,j}(x) Y_{k,j}(y) = N(d,k) P_k(x^Ty)
\end{equation}
\cite[Theorem~2.24]{atkinson2012spherical}.
For any $f:[-1,1]\mapsto\RR$, $x \in \bS^{d-1}$ and  $Y_k\in \cY_k^d$, the Hecke-Funk formula \citep[Theorem~2.22]{atkinson2012spherical} is given by 
\begin{equation}\label{eqn: hecke-funk}
\int_{\mathbb{S}^{d-1}} f(x^{\top} y) Y_{k}(y) \dd \tau_{d-1}(y)=\frac{\omega_{d-2}}{\omega_{d-1}} Y_{k}(x) \int_{-1}^{1} f(t) P_{k}(t)\left(1-t^{2}\right)^{(d-3) / 2} \dd t.
\end{equation}
We refer to \cite{schoenberg1988positive} and \cite{atkinson2012spherical} for more details about the harmonic analysis on $\SS^{d-1}$.

%Let $Y_{k,j}:\SS^{d-1}\mapsto\RR$ denotes the $j$-th spherical harmonics of degree $k$. Then,  $k\geq 0$ and $1\leq j\leq N(d, k)$, and $\{Y_{k,j}\}$   forms an orthnormal basis of $L^2(\tau_{d-1})$: $\int_{S^{d-1}} Y_{k,j}(x)Y_{k',j'}(x)d\tau_{d-1}(x) = \delta_{k,k'}\delta_{j,j'}$. The spherical harmonics is related to the Legendre polynomials: 
%\begin{equation}\label{eqn: sum_of_Y}
%    \sum_{j=1}^{N(d,k)} Y_{k,j}(x) Y_{k,j}(y) = N(d,k) P_k(x^Ty).
%\end{equation}
%For any $f:[-1,1]\mapsto\RR$ and  $Y_k\in \text{span}\{Y_{k,1},\dots,Y_{k,N(d,k)}\}$, the Hecke-Funk formula is given by 
%\begin{equation}\label{eqn: hecke-funk}
%\int_{\mathbb{S}^{d-1}} f(x^{\top} y) Y_{k}(y) \dd \tau_{d-1}(y)=\frac{\omega_{d-2}}{\omega_{d-1}} Y_{k}(x) \int_{-1}^{1} f(t) P_{k}(t)\left(1-t^{2}\right)^{(d-3) / 2} \dd t.
%\end{equation}
%We refer to \citep{schoenberg1988positive} for more details about function analysis on $\SS^{d-1}$.

\section{A general result}
Denote by $\varphi:X\times \Omega \mapsto \RR$ a general parametric feature. For any $\pi\in\cP(\Omega)$, define
\begin{equation}\label{eqn: kernel-definition}
\begin{aligned}
   k_\pi: X\times X\mapsto\RR,\qquad k_\pi(x,x') = \langle \varphi(x;\cdot), \varphi(x';\cdot)\rangle_\pi.
\end{aligned}
\end{equation}
Assume that $\varphi(\cdot;v)$ is continuous on $X$ for any $v\in\Omega$ and $X$ is compact. By Mercer's theorem, we have the eigendecomposition:
$
    k_\pi(x,x') = \sum_{j=1}^\infty \lambda_j e_j(x)e_j(x'),
$
where $\{\lambda_j\}_{j\geq 1}$ are the eigenvalues in a non-increasing order and $\{e_j\}_{j\geq 1}$ are the corresponding eigenfunctions that satisfy 
$
\EE_{x'\sim\rho}[k_\pi(x,x')e_j(x')]=\lambda_j e_j(x).
$
The trace of $k_\pi$  satisfies that $\sum_{j=1}^\infty \lambda_j = \EE_{x\sim\rho}[k_\pi(x,x)]$.  In particular, we are interested in the following quantity
\begin{align}\label{eqn: trace-decay}
    \Lambda_{\pi}(m) = \sum_{j=m+1}^\infty \lambda_j.
\end{align}

  Consider the model: 
$
	g_m(x;\theta) = \sum_{j=1}^m a_j\varphi(x;v_j),
$
where $\theta=\{(a_j,v_j)\}_{j=1}^m$ denote the learnable parameters. In the literature, this model is often called \emph{variable-basis approximation} \citep{gnecco2012comparison,kurkova2001bounds}. Two-layer neural  networks correspond to the specical case where $\varphi(x;v)=\sigma(w^Tx+b)$. 
Other examples include the free-node splines \citep[Chapter 11]{devore1993constructive}, trigonometric polynomials with free frequencies \citep{devore1995nonlinear}, etc. 

Recall that $\Phi=\{\varphi(\cdot;v)\,|\, v \in \Omega\}$ is the class of parametric feature functions. Define
\begin{equation}\label{eqn: continuous-formulation}
\cG = \overline{\left\{ \sum_{j=1}^m a_j\varphi(\cdot;v_j)\, \big|\, \sum_{j=1}^m |a_j|\leq 1, m\in\mathbb{N}^{+}\right\}},
\end{equation}
which is just the closure of the convex, symmetric hull of $\Phi$.  When $\varphi(x;v)=\sigma(w^Tx+b)$ and $\Omega=\{(w^T,b)^T\in\RR^{d+1}: \|w\|_2+|b|\leq r\}$, we have $\cG=\cN$.

The following theorem provides a lower bound of the Kolmogorov width of $\cG$.
\begin{proposition}[Spectral-based lower bound]\label{thm: 1}
$
\omega_m(\cG) \geq \sup_{\pi\in\cP(\Omega)}\Lambda_\pi(m).
$
\end{proposition}

The proof of Proposition \ref{thm: 1} needs the following characterization of the average Kolmogorov width of $\Phi$.
\begin{lemma}\label{pro: gen-feature}
$
	A_m^{(\pi)}(\Phi)=\EE_{v \sim \pi} \inf_{c_1,\dots,c_m \in \bR}\|\varphi_v - \sum_{j=1}^{m}c_j e_j\|_{\rho}^2 =\Lambda_\pi(m),
$
where $\varphi_v = \varphi(\cdot;v)$.
\end{lemma}

It is implied that the best-possible features minimizing the average approximation error are the leading eigenfunctions.

\begin{proof} 
Without loss of generality (WLOG), assume $\{\phi_j\}_{j=1}^m$ are orthonormal in $L^2(\rho)$, otherwise we can perform Gram-Schmidt orthonormalization. Then,

\begin{align}\label{pca_property}
    \EE_{v \sim \pi} &\inf_{c_1,\dots,c_m \in \bR}\|\varphi_v - \sum_{j=1}^{m}c_j \phi_j\|_{\rho}^2 = \EE_{v \sim \pi} \|\varphi_v - \sum_{j=1}^{m}\langle\varphi_v,\phi_j \rangle_{\rho} \phi_j\|_{\rho}^2 = \EE_{v \sim \pi}\big[\|\varphi_v\|_{\rho}^2 -\sum_{j=1}^m \langle\varphi_v,\phi_j\rangle_{\rho}^2\big] \notag \\
\notag    &= \EE_{x\sim\rho} [k_\pi(x,x)] - \sum_{j=1}^{m}\EE_{x,x'\sim\rho}[\phi_j(x)\phi_j(x')k(x,x')]\\
    &\geq \sum_{j=1}^\infty \lambda_j - \sup_{\langle \psi_j,\psi_{j'}\rangle_\rho=\delta_{j,j'}}\sum_{j=1}^{m}\EE_{x,x'\sim\rho}[\psi_j(x)\psi_j(x')k_\pi(x,x')].
\end{align}
The second term in \eqref{pca_property} is a standard PCA problem for  $k_\pi$, where the supremum is reached at $\phi_j=e_j$ and 
$
\sum_{j=1}^{m}\EE_{x,x'\sim\rho}[e_j(x)e_j(x')k_\pi(x,x')] = \sum_{j=1}^m \lambda_j.
$
Plugging it back to \eqref{pca_property} completes the proof.
\end{proof}

\paragraph*{Proof of Proposition \ref{thm: 1}} 
Lemma \ref{pro: gen-feature} provides a estimate of the average approximation error. The Kolmogorov width, which is the worst-case error, can be bounded as follows:
\begin{equation}\label{eqn: worst}
    \omega_m(\Phi)=\inf_{\phi_1,\dots,\phi_m} \sup_{v\in\Omega} \inf_{c_1,\dots,c_m \in \bR}\|\varphi_{v} - \sum_{j=1}^{m}c_j\phi_j \|_{\rho}^2 \ge A_{m}^{(\pi)}(\Phi)= \Lambda_{\pi}(m), \quad \forall \pi \in \cP(\Omega).
\end{equation}
Hence, 
$
\omega_m(\cG)\geq  \omega_m(\Phi)\geq  \sup_{\pi\in\cP(\Omega)}\Lambda_\pi(m) 
$
because of $\Phi\subset \cG$.
$\qed$

% Let $\cT: L^2(\mu)\mapsto L^2(\rho)$ the integral operator: $\cT a=\int a(v)\varphi(\cdot;v)\dd\mu(v)$. The Hilbert-Schmidt condition  ensures $\cT$ to be  compact. Notice that $\cH_\mu$ is the image of $\cT$ applying to the unit ball of  $L^2(\mu)$. Then
% by \cite[Theorem 2.2]{pinkus2012n}, we have $\omega_m(\cH_\mu)=\lambda_{m+1}^{(\mu)}$.

The spectral-based lower bound of $\omega_m(\cG)$ in Proposition \ref{thm: 1} is quite general, holding for any parametric features. Even when explicit estimates of the eigenvalues of $k_\pi$ are not tractable, we can still numerically compute ones, thereby obtaining a numerical lower bound of $\omega_m(\cG)$.

We notice that a similar lower bound of Kolmogorov width already appeared in \cite{ismagilov1968n} for a completely different purpose. At first glance, the function class considered  in \cite{ismagilov1968n} may look different from $\cG$, but they are in fact equivalent (see \cite[Section 3.4]{pinkus2012n}). However, to the best of our knowledge, our work is the first one exploiting this approach to study variable-basis approximations and two-layer neural networks. In contrast, the existing works on bounding $\omega_m(\cG)$  \citep{kurkova2002comparison,gnecco2012comparison,siegel2021sharp}  all rely on the  orthogonal function argument, whose applicability is limited to some special cases. Moreover, we  obtain an upper bound that match the lower bound in Proposition \ref{thm: 1} for two-layer neural networks.

In the remaining of this paper, we will apply the above general result to the case of two-layer neural networks and the input distribution $\rho=\tau_{d-1}$. In such case, we can have explicit estimates of the eigenvalues, thereby the Kolmogorov widths. Moreover, taking $\pi=\tau_{d-1}$ 
also yields an upper bound that matches the above lower bound.  For simplicity,  we will omit the subscript of $\Lambda_\pi$ since $\pi$ is always fixed to be $\tau_{d-1}$ in the following analyses.

\section{Approximation of single neurons}
\label{sec: single-neuron}

We first consider the single neuron without bias: $x\mapsto \sigma(v^Tx)$.  The results established in this section will be utilized later to analyze two-layer neural networks.

Assume that $\pi=\rho=\tau_{d-1}$. By the rotational invariance,  $k_\pi$ can be written in a dot-product form:
\begin{equation}\label{eqn: dot-product-def}
    k_\pi(x,x')  = \int_{\SS^{d-1}} \sigma(v^Tx)\sigma(v^Tx') \dd \tau_{d-1}(v)=\kappa(x^Tx')
\end{equation}
where $\kappa: [-1,1]\mapsto\RR$. Following \cite{smola2001regularization}, the spectral decomposition of $\kappa$ is given by:
\begin{equation}\label{eqn: def_mu}
    \kappa(x^Tx') = \sum_{k=0}^\infty\sum_{j=1}^{N(d,k)} \mu_k Y_{k,j}(x)Y_{k,j}(x'),
\end{equation}
where $\mu_k$ is the eigenvalue and the spherical harmonics $Y_{k,j}$ is the corresponding eigenfunction that satisfies  
$\EE_{x'\sim\tau_{d-1}}[\kappa(x^Tx')Y_{k,j}(x')]=\mu_k Y_{j,k}(x)$. Note that $\{\lambda_j\}$ are the eigenvalues counted with multiplicity, while $\{\mu_k\}$ are the eigenvalues counted without multiplicity.  We refer to \citep{schoenberg1988positive,smola2001regularization} for more details about the spectral decomposition of a dot-product kernel on $\SS^{d-1}$.

Applying Lemma \ref{pro: gen-feature} to single neurons gives	the following lower bound 
\begin{equation}
    \EE_{v\sim\tau_{d-1}}\inf_{c_1,\dots,c_m} \|\sigma_v - \sum_{j=1}^m c_j \phi_j\|_{\tau_{d-1}}^2 \geq \Lambda(m),
\end{equation}
and the inequality can be achieved by using the spherical harmonics as the fixed features.

By exploiting the rotational invariance, one can show the following upper bound. 
\begin{proposition}[Uniform approximability]\label{pro: smooth-pointwise}
Let $\{\phi_j\}_{j=1}^{m}$ be the leading spherical harmonics. 
For any non-increasing function $L: \mathbb{N}^{+} \rightarrow \bR^+$ that satisfies $\Lambda(m)\le L(m) $, let $q(d,L) = \sup_{k \ge 1} \frac{L(k)}{L((d+1)k)}$. Then we have for any $v\in \SS^{d-1}$,
\begin{equation}\label{eqn: smooth-ub-0}
\inf_{c_1,\dots,c_{m}} \|\sigma_v - \sum_{j=1}^{m}c_j \phi_j\|_{\tau_{d-1}}^2 \lesssim q(d,L)L(m).
\end{equation}
\end{proposition}
The proof is deferred to Appendix \ref{sec: single-neuron-smooth-pointwise}. 
When $L(m)\sim m^{-s}$, $q(d,L)\lesssim d^{s}$. Therefore, $q(d,L)$ is at most polynomial in $d$.
In fact, one can  choose $L(m)=\Lambda(m)$ to obtain the tightest bound. The introduction of $L(m)$ is to facilitate the explicit calculation of constant $q(d,L)$, since we hardly have the exact rate of $\Lambda(m)$.

By applying the eigenvalue estimates given in the next subsections, we have the following specific results.
For the nonsmooth activations considered in Section \ref{sec: nonsmooth-single-neuron}, taking $L(m)\sim m^{-\frac{2\alpha+1}{d-1}}$ yields $q(d,L)\lesssim d^{\frac{2\alpha+1}{d-1}}\leq C_\alpha$ with $C_\alpha$ being a positive constant independent of $d$.  For smooth activation functions  that satisfy Assumption \ref{assumption: smooth-activation}, we can take $L(m)=1/m$, for which $q(d,L)\lesssim d$.

What remains is to estimate the eigenvalues of $k_\pi$ and the following integral representation allows both explicit estimations and numerical computations of the eigenvalues.
\begin{lemma}\label{lemma: eigen}
For the dot-product kernel $k_\pi$,  $\mu_k=\eta_k^2$ with 
\begin{align}\label{eqn: eigen-sv}
    \eta_k = \frac{\omega_{d-2}}{\omega_{d-1}}\int_{-1}^1 \sigma(t)P_k(t)(1-t^2)^{(d-3)/2} \dd t.
\end{align}
In addition, assume that $\sigma\in C^{\infty}(\RR)$. Then,
\begin{equation}\label{eqn: smooth-eta_k}
    \eta_k = \frac{\Gamma(d/2)}{2^k\sqrt{\pi}\Gamma(k+(d-1)/2)}\int_{-1}^1 \sigma^{(k)}(t)\left(1-t^{2}\right)^{k+(d-3) / 2} \dd t.
\end{equation}
\end{lemma}
\begin{proof}
By the Hecke-Funk formula \eqref{eqn: hecke-funk}, 
\begin{align*}
\int_{\SS^{d-1}} \kappa(x^Tx') Y_{k,j}(x') d\tau_{d-1}(x') &= \int_{\SS^{d-1}} \left(\int_{\SS^{d-1}}\sigma(w^Tx)\sigma(w^Tx')\dd\tau_{d-1}(w)\right) Y_{k,j}(x') \dd\tau_{d-1}(x')\\
&= \int_{\SS^{d-1}}\sigma(w^Tx) \left(\int_{\SS^{d-1}}\sigma(w^Tx') Y_{k,j}(x') \dd\tau_{d-1}(x')\right) \dd\tau_{d-1}(w)\\
&=\int_{\SS^{d-1}}\sigma(w^Tx) \eta_k Y_{k,j}(w) \dd\tau_{d-1}(w) = \eta_k^2 Y_{k,j}(x).
\end{align*}
Substituting the  Rodrigues formula \eqref{eqn: Rodrigues} into \eqref{eqn: eigen-sv} and applying the  integration by parts give 
\begin{align*}
\notag \eta_k &= \left(-\frac{1}{2}\right)^k \frac{\omega_{d-2}}{\omega_{d-1}}\frac{\Gamma((d-1)/2)}{\Gamma(k+(d-1)/2)} \int_{-1}^1 \sigma(t) \left(\frac{d}{d t}\right)^{k}\left(1-t^{2}\right)^{k+(d-3) / 2} \dd t \\
 &= \frac{1}{2^k} \frac{\omega_{d-2}}{\omega_{d-1}}\frac{\Gamma((d-1)/2)}{\Gamma(k+(d-1)/2)} \int_{-1}^1 \sigma^{(k)}(t)\left(1-t^{2}\right)^{k+(d-3) / 2} \dd t.
\end{align*}
Inserting $\omega_{d-1}=\frac{2\pi^{d/2}}{\Gamma(d/2)}$ completes the proof.
\end{proof}

We remark that the integral representation \eqref{eqn: eigen-sv} has been adopted in \cite[Appendix D]{bach2017breaking} to calculate the eigenvalues of $k_\pi$ for  ReLU$^\alpha$ activations. Eq.~\eqref{eqn: smooth-eta_k} also follows straightforwardly from Eq.~\eqref{eqn: eigen-sv}.  We provide the proof here since it is simple but very helpful for understanding what property of activation functions affect the eigenvalue. In particular, Eq.~\eqref{eqn: smooth-eta_k} shows that the smaller is the $k$-th order derivative, the smaller is the eigenvalue and this formula
will be used later to estimate the eigenvalues for smooth activation functions. 
% In addition, Eq.~\eqref{eqn: eigen-sv} provides  compute the eigenvalues numerically.

\paragraph*{Numerical computation of $\Lambda(m)$}  The following procedures provide a numerical way to compute $\Lambda(m)$, which works for any activation functions.
\begin{itemize}
\item First,
$
\sum_{j=1}^\infty\lambda_j=\int_{\SS^{d-1}}\kappa(x^Tx)\dd\tau_{d-1}(x) = \kappa(1).
$
If $\kappa(1)$ does not have an explicit expression, we can use Monte-Carlo integration to numerically compute it by Eq.~\eqref{eqn: dot-product-def}.
\item Second, the eigenvalues $\{\lambda_j\}$ are computed through numerically integrating the right hand side of Eq.~\eqref{eqn: eigen-sv}, where the Legendre polynomials can be efficiently computed using the recursive formula \eqref{eq: recursive}. 
\item The output is: $\Lambda(m)=\kappa(1)-\sum_{j=1}^m \lambda_j$.
\end{itemize}

\subsection{Nonsmooth activations}
\label{sec: nonsmooth-single-neuron}
Consider the ReLU$^\alpha$ activation function: $\sigma(t)=\max(0,t)^\alpha$ with $\alpha\in \mathbb{N}$.  The Heaviside step and ReLU function correspond to $\alpha=0$ and $\alpha=1$, respectively. 
The case of $\alpha>1$ also has many applications in scientific computing \citep{weinan2018deep,siegel2020high,li2019better}. In particular, for $\alpha=0,1$,  \cite{cho2009kernel} shows
\begin{align}\label{eqn: cosine-kernel}
\kappa(t) = 
\begin{cases}
\frac{1}{2\pi}(\pi - \arccos(t)) & \text{if } \alpha=0 \\
\frac{1}{2\pi d}\left((\pi-\arccos(t))t + \sqrt{1-t^2}\right) & \text{if } \alpha=1.
\end{cases}
\end{align}

\begin{proposition}\label{pro: non-smooth-1}
Let $\sigma(t)=\max(0,t)^\alpha$ with $\alpha\in \mathbb{N}$. There exists a constant $C(1/d,\alpha)$ depending on $1/d$ polynomially such that 
$
    \Lambda(m)\geq C(1/d,\alpha) m^{-\frac{2\alpha+1}{d-1}}.
$ In particular, $C(1/d,0)=1/d$.
\end{proposition}

The proof is quite techinical and deferred to Appendix \ref{sec: single-non-smooth}, where the analytic expression of the eigenvalue $\mu_k$ obtained in \cite[Appendix D]{bach2017breaking} is used.
Figure \ref{fig: approx} compares the above bounds of $\Lambda(m)$ and numerical estimations for various $d$'s for the case of $\alpha=0$. It is clear that the decay suffers from the CoD and the explicit rate $m^{-(2\alpha+1)/(d-1)}$ given in Proposition \ref{pro: non-smooth-1} aligns very well with the ground truth for large $m$'s.

\begin{figure}[!h]
    \centering
    \includegraphics[width=.4\textwidth]{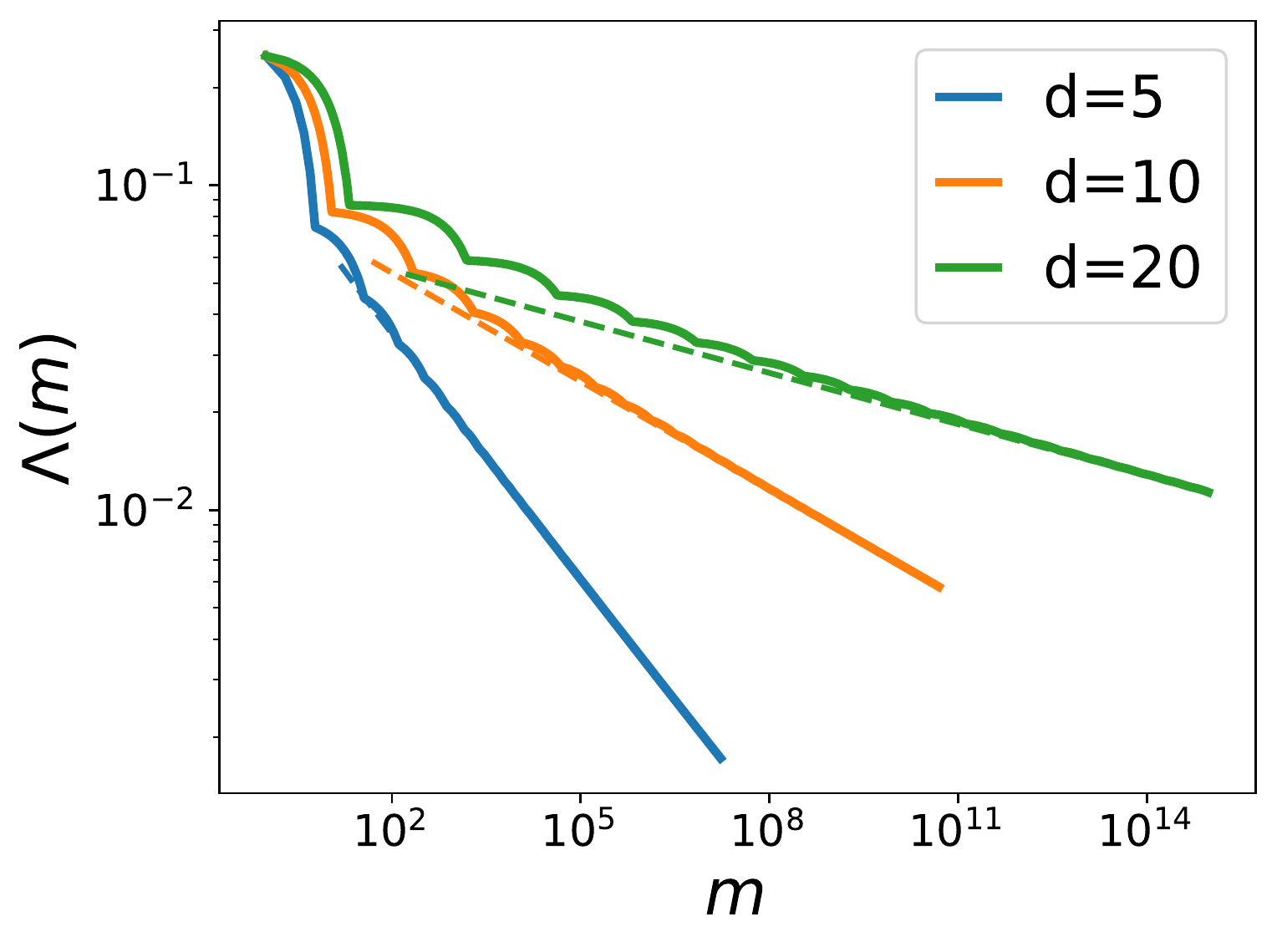}
    \vspace{-3mm}
    \caption{The decay of $\Lambda(m)$ for various  $d$'s and $\alpha=0$. The solid curve corresponds to  the numerical estimate, while the dashed curve corresponds the explicit estimate $m^{-(2\alpha+1)/(d-1)}$ given in Proposition \ref{pro: non-smooth-1}.
    }
    \label{fig: approx}
\end{figure}

Then we have the following theorem, which shows that approximating single neurons activated by  ReLU$^\alpha$ with linear methods suffers from the CoD. 
\begin{theorem}\label{thm: gen-1}
Let $\sigma_v(x)=\max(0,v^Tx)^\alpha$ for $\alpha\in \mathbb{N}$. Then, there exists a constant $C(1/d,\alpha)$ that depends on $1/d$ polynomially such that the following statements hold. 
\begin{itemize}
\item For any fixed features $\{\phi_j\}_{j=1}^m$, we have
\begin{equation}\label{eqn: thm-11}
       \EE_{v \sim \tau_{d-1}} \inf_{c_1,\dots,c_m \in \bR}\|\sigma_{v} - \sum_{j=1}^{m}c_j\phi_j \|_{\tau_{d-1}}^2 \ge C(1/d,\alpha) m^{-\frac{2\alpha+1}{d-1}}.
\end{equation}
\item Consider the random feature: $f_j(\cdot;\cdot):\RR^d\times \RR^{d\times q_j}\mapsto \RR$. We assume $\{f_j\}_{j=1}^m$ are rotationally invariant, i.e., for any $j\in [m]$, $f_j(x;W_j)=f_j(Qx;QW_j)$ for any orthonormal matrix $Q\in\RR^{d\times d}$ and $W_j$ is sampled from a rotation-invariant distribution $\pi_j$. Let $\bar{W}=(W_1,\dots,W_m)$. Then, for any $v\in\SS^{d-1}$, 
\begin{equation}\label{eqn: thm-22}
\EE_{\bar{W}}\inf_{c_1,\dots,c_m} \|\sigma_{v}-\sum_{j=1}^mc_j f_j(\cdot;W_j)\|_{\tau_{d-1}}^2 \geq C(1/d,\alpha) m^{-\frac{2\alpha+1}{d-1}}.
\end{equation}
If $m\leq 2^d$, with probability larger than $C(1/d,\alpha)$ over the sampling of $\bar{W}$, we have  for any $v\in\SS^{d-1}$,
\begin{equation}\label{eqn: thm-33}
\inf_{c_1,\dots,c_m} \|\sigma_{v}-\sum_{j=1}^mc_j f_j(\cdot;W_j)\|_{\tau_{d-1}}^2 \geq C(1/d,\alpha).
\end{equation}
\end{itemize}
\end{theorem}

Eq.~\eqref{eqn: thm-11} shows that the average error of approximating single neurons with any fixed features suffers from the CoD. This suggests that there  exists  a single neuron $\sigma_v$ such that the approximation is difficult, whereas the specific choice of $v$ is unknown and it depends on the features. Eq.~\eqref{eqn: thm-22} and  \eqref{eqn: thm-33} improve this by showing that the same conclusion holds for any $v\in\SS^{d-1}$  as long as the features are rotationally invariant. Specifically, Eq.~\eqref{eqn: thm-22} bounds the error in expectation with respect to the sampling of random features and Eq.~\eqref{eqn: thm-33} further provides a bound of the failure probability.

A typical form of rotation-invariant random features is $f_j(x;W_j)=g_j(W_j^Tx)$ with $g_j:\RR^{q_j}\mapsto \RR^1$, which include  $f_j(x;w)=\sigma(w^Tx)$ (emerging in analyzing neural networks) and kernel predictors with dot-product kernels. Therefore, our results cover the setting considered in \cite{yehudai2019power} but successfully remove all the limitations of \cite{yehudai2019power}. Specifically, we impose no restriction on the coefficient magnitudes and do not need to adversarially choose the bias term. 
Note that any single neuron can be approximated exactly by two-layer neural network of width $m=1$, whereas the random feature approximation requires exponentially many features. They together provides a CoD-type separation between two methods for approximating this specific function.

\paragraph*{Proof of Theorem \ref{thm: gen-1}}
Eq.~\eqref{eqn: thm-11}  follows from a simple combination of Proposition \ref{pro: non-smooth-1} and Lemma \ref{pro: gen-feature}.  To prove Eq.~\eqref{eqn: thm-22}, we need to exploit the rotational invariance of the random features. 
Let  $S_v(\bar{W})=\inf_{c_1,\dots,c_m}\|\sigma_v - \sum_{j=1}^m c_j  f_j(\cdot;W_j)\|_{\tau_{d-1}}^2$. For any $v\in\SS^{d-1}$, let $Q_v\in\RR^{d\times d}$ be an orthonormal matrix such that $Q v=e_1$. Then,
\begin{align*}
    \EE_{\bar{W}}[S_v(\bar{W})] &= \EE_{\bar{W}}\inf_{c_1,\dots,c_m}\|\sigma_v - \sum_{j=1}^m c_j  f_j(\cdot;W_j)\|_{\rho}^2\stackrel{(i)}{=}\EE_{\bar{W}}\inf_{c_1,\dots,c_m}\|\sigma_{e_1} - \sum_{j=1}^m c_j  f_j(\cdot;Q_vW_j)\|_{\rho}^2\\
    &\stackrel{(ii)}{=} \EE_{\bar{W}}\inf_{c_1,\dots,c_m}\|\sigma_{e_1} - \sum_{j=1}^m c_j  f_j(\cdot;W_j)\|_{\rho}^2=\EE_{\bar{W}}[S_{e_1}(\bar{W})],
\end{align*}
where $(i)$ and $(ii)$ follow from the rotational invariance of $\rho$ and $\{\pi_j\}$, respectively. Therefore, $\EE_{\bar{W}}[S_v(\bar{W})]$ is constant with respect to $v$.  By Lemma \ref{pro: gen-feature}, we have 
\begin{align*}
    \EE_{\bar{W}}[S_v(\bar{W})] &= \EE_{v\sim\tau_{d-1}} \EE_{\bar{W}}[S_v(\bar{W})] = \EE_{\bar{W}} \EE_{v\sim\tau_{d-1}}[S_v(\bar{W})] \geq \EE_{\bar{W}}[\Lambda (m)] = \Lambda(m).
\end{align*}
Then, applying Proposition \ref{pro: non-smooth-1} completes the proof of Eq.~\eqref{eqn: thm-22}.

In addition, we have  
\begin{align*}
S_v(\bar{W})&\leq \|\sigma_v\|_{\tau_{d-1}}^2= \int_{\SS^{d-1}}\sigma(v^Tx)^2 \dd \tau_{d-1}=\int_0^1 t^{2\alpha} p_d(t)\dd t=\frac{\Gamma(d/2)}{\Gamma(d/2+\alpha)} \lesssim \frac{1}{d^\alpha},
\end{align*}
where $p_d(t)=\frac{(1-t^2)^{\frac{d-3}{2}}}{B(\frac 1 2,\frac{d-1}{2})}$ is the density function of $v^Tx$.

Then, 
\begin{equation}
\Lambda(m)\leq \EE[S_v(\bar{W})] \lesssim \PP\{S_v(\bar{W})\geq t\} d^{-\alpha} + \PP\{S_v(\bar{W})\leq t\} t\leq \PP\{S_v(\bar{W})\geq t\} d^{-\alpha} + t.
\end{equation}
Taking $t=0.5\Lambda(m)$, we have 
\[
\PP\{S_v(\bar{W})\geq 0.5\Lambda(m)\}\gtrsim \frac{d^\alpha}{2}\Lambda(m).
\]
When $m\leq 2^d$, $\Lambda(m)\geq C(1/d,\alpha) 2^{-2\alpha-1}$. Hence, we complete the proof of Eq.~\eqref{eqn: thm-33}.

$\qed$
% Note that  
% Eq.~\eqref{eqn: thm-11} simply follows from Proposition \ref{pro: gen-feature} and Proposition \ref{pro: non-smooth-1}. However, to prove Eq.~\eqref{eqn: thm-22}, we must exploit the rotational invariance of $\sigma_v$ and the random features. For details, we refer to Section \ref{sec: proof-thm-1}. 

\subsection{Smooth activations}
\label{sec: smooth-single-neuron}

We now turn to smooth activation functions, such as sigmoid,  softplus, arctan, GELU, and Swish/SiLU.
%\begin{lemma}\label{lemma: smooth-eta_k}
%Assume that $\sigma\in C^{\infty}(\RR)$. Then,
%\begin{equation}\label{eqn: smooth-eta_k}
%    \eta_k = \frac{\Gamma(d/2)}{2^k\sqrt{\pi}\Gamma(k+(d-1)/2)}\int_{-1}^1 \sigma^{(k)}(t)\left(1-t^{2}\right)^{k+(d-3) / 2} \dd t.
%\end{equation}
%\end{lemma}
%\begin{proof}
%Substituting the  Rodrigues formula \eqref{eqn: Rodrigues} into Eq.~\eqref{eqn: eigen-sv} leads to  
%\begin{align*}
%\notag \eta_k &= \left(-\frac{1}{2}\right)^k \frac{\omega_{d-2}}{\omega_{d-1}}\frac{\Gamma((d-1)/2)}{\Gamma(k+(d-1)/2)} \int_{-1}^1 \sigma(t) \left(\frac{d}{d t}\right)^{k}\left(1-t^{2}\right)^{k+(d-3) / 2} \dd t \\
 %&= \frac{1}{2^k} \frac{\omega_{d-2}}{\omega_{d-1}}\frac{\Gamma((d-1)/2)}{\Gamma(k+(d-1)/2)} \int_{-1}^1 \sigma^{(k)}(t)\left(1-t^{2}\right)^{k+(d-3) / 2} \dd t,
%\end{align*}
%where the last equality follows from integration by parts. Inserting $\omega_{d-1}=\frac{2\pi^{d/2}}{\Gamma(d/2)}$ completes the proof.
%\end{proof}
We first have the following lemma, which bounds the eigenvalue  $\mu_k$ by using the $k$-th order derivative of $\sigma$.
\begin{lemma}\label{lemma: smooth-activation}
Assume that $\sigma$ is smooth and let $B_k=\sup_{t\in \RR}|\sigma^{(k)}(t)|$. Then, we have 
\[
\mu_k \leq \frac{B_k^2}{2^{2k}} \frac{\Gamma(d/2)^2}{\Gamma(k+d/2)^2}.
\]
\end{lemma}

\begin{proof}
By the assumption, 
\begin{align}\label{eqn: smooth-2}
\notag \big|\int_{-1}^1 \sigma^{(k)}(t)&\left(1-t^{2}\right)^{k+(d-3) / 2} \dd t\big| \leq B_k \int_{-1}^1 \left(1-t^{2}\right)^{k+(d-3) / 2} \dd t\\
&= B_k \int_{0}^1 u^{-1/2} (1-u)^{k+(d-3)/2} \dd u = B_k \frac{\Gamma(1/2)\Gamma(k+(d-1)/2)}{\Gamma(k+d/2)},
\end{align}
where the second equality follows from the change of variable $u=t^2$. 
Then, using Eq.~\eqref{eqn: smooth-eta_k} and $\omega_{d-1} = \frac{2\pi^{d/2}}{\Gamma(d/2)}$ gives rise 
\begin{align*}
|\eta_k|\leq \frac{B_k}{2^k}\frac{\omega_{d-2}}{\omega_{d-1}}\frac{\Gamma(1/2)\Gamma((d-1)/2)}{\Gamma(k+d/2)} = \frac{B_k}{2^k}\frac{\Gamma(d/2)}{\Gamma(k+d/2)}.
\end{align*}
Then, applying $\mu_k = \eta_k^2$ completes the proof.
\end{proof}

\begin{assumption}\label{assumption: smooth-activation}
Assume that $B_k:=\max_{t\in\RR}|\sigma^{(k)}(t)|\lesssim \Gamma(k+1)$.
\end{assumption}
All the popular smooth activation functions satisfy the above assumption as shown below.
\begin{itemize}
\item For $\sigma(t)=\sin(t)$ and $\sigma(t)=\cos(t)$, $B_k=1$. 
\item Consider the sigmoid function: $\sigma(z)=1/(1+e^{-z})$, which can be viewed as a complex function $\mathbb{C}\mapsto\mathbb{C}$. The singular points of $\sigma$ are $\{z=(2k+1)\pi i\}_{k\in\mathbb{Z}}$. For any $t\in\RR$, let $C_{t} = \{z\in \mathbb{C}\,:\,|z-t|=2\}$. Then, all the singular points must be outside the curve $C_t$. Using Cauchy's integral formula, for any $t\in\RR$, we have
\begin{align}\label{eqn: Cauchy}
\notag |\sigma^{(k)}(t)| &= \left|\frac{\Gamma(k+1)}{2\pi i}\int_{C_{t}} \frac{\sigma(z)}{(z-t)^{k+1}}\dd z\right|\leq\frac{\Gamma(k+1)}{2\pi}\int_{C_{t}} \frac{|\sigma(z)|}{|z-t|^{k+1}} |\dd z|\\
& \leq \frac{\Gamma(k+1)\max_{z\in C_t}|\sigma(z)|}{2\pi 2^{k+1}}\int_{C_{t}} |\dd z|\leq \frac{\Gamma(k+1)\max_{z\in C_t}|\sigma(z)|}{2^k} \lesssim \frac{\Gamma(k+1)}{2^k}.
\end{align}
\item For all the other commonly-used smooth activation functions, we can obtain similar estimates of the $k$-th order derivatives by using Cauchy's integral formula.
\end{itemize}

\begin{proposition}\label{pro: smooth-lower-bound}
Under Assumption \ref{assumption: smooth-activation},  we have 
$
    \Lambda(m)\lesssim 1/m.
$
\end{proposition}

The proof is deferred to Appendix \ref{sec: proof-single-neuron-smooth-trace}. We remark that the above estimate of $\Lambda(m)$ is rather rough  for most smooth activation functions, where $B_k$ is much smaller than $\Gamma(k+1)$ as demonstrated in  Eq.~\eqref{eqn: Cauchy}. A simple combination of Proposition \ref{pro: smooth-lower-bound} and the proof of Lemma \ref{pro: gen-feature} gives  

\begin{equation}\label{eqn: smooth-err}
    \EE_{v\sim\tau_{d-1}}\inf_{c_1,\dots,c_m} \|\sigma_v - \sum_{j=1}^m c_j \phi_j\|_{\tau_{d-1}}^2 \lesssim \frac 1 m,
\end{equation}
where $\{\phi_j\}_{j=1}^{m}$ are the leading spherical harmonics. Applying Proposition \ref{pro: smooth-pointwise}  to activation functions satisfying Assumption \ref{assumption: smooth-activation}, we can obtain that
\begin{equation}
    \sup_{v \in \bS^{d-1}}\inf_{c_1,\dots,c_m} \|\sigma_v - \sum_{j=1}^m c_j \phi_j\|_{\tau_{d-1}}^2 \lesssim \frac d m,
\end{equation}
By comparing with  Theorem \ref{thm: gen-1}, we see that for smooth activations, the approximation with fixed features does not suffer from the CoD. This is very different from the nonsmooth ones.

\section{Kolmogorov widths of two-layer neural networks}
\label{sec: two-layer-network}
We are now ready to estimate $\omega_m(\cN)$, which describes the (in)approximability of $\cN$ by linear methods. In this section, we use $\cN^r$ instead of $\cN$ for emphasizing the dependence on the norms of inner-layer widths.  In addition, in order to deal with the bias term, we  define
$\sigma^{(\gamma,b)}(t)=\sigma(\gamma t+b)$ for $\gamma>0,b\in\RR$ and the associated kernel 
\begin{equation}
k^{(\gamma,b)}(x,x')=\EE_{v\sim\tau_{d-1}}[\sigma(\gamma v^Tx+b)\sigma(\gamma v^Tx'+b)]. 
\end{equation}
Let $\Lambda^{(\gamma,b)}(\cdot)$ denote the trace decay of $k^{(\gamma,b)}$ defined according to  Eq.~\eqref{eqn: trace-decay}. 

% To analyze the expressiveness of the two-layer neural network, we define 
% \begin{equation}
% \cB_{\sigma, r} := \left\{ \int_{\Omega_r} a(w,b) \sigma(w^Tx+b) \dd\pi(w,b): \pi\in \cP(\Omega_r),\, \EE_{(w,b)\sim\pi}[|a(w,b)|]<\infty\right\},
% \end{equation}
% where $\Omega_r = \{(w,b)\in\RR^{d+1}: \|w\|_2+|b|\leq r\}$. 
% For any $f\in \cB_{\sigma,r}$, let
% \begin{equation}
%     \|f\|_{\cB_{\sigma, r}} := \inf_{\EE_{(w,b)\sim\pi}[a(w,b)\sigma(w^Tx+b)]=f(x)} \EE_{(w,b)\sim\pi}[|a(w,b)|].
% \end{equation}
% Obviously, all the finite-width neural network belongs to $\cB_{\sigma, r}$.
%  We use $\BB_{\sigma,r}$ to denote the unit ball of $\cB_{\sigma,r}$.
% Similar function spaces have been widely studied previously and play a fundamental role in theoretical analysis of two-layer neural networks \citep{bach2017breaking,ma2019priori,weinan2021barron,chen2020dynamical,chizat2020implicit,lu2021priori,wojtowytsch2020some}. Our definition is slightly different from the ones in \citep{bach2017breaking,weinan2021barron}, which are specific to  ReLU.

\begin{theorem}\label{proposition: 2lnn-complete-characterization}
Let $\Lambda_{r}(m)=\sup_{\gamma+|b|\leq r}\Lambda^{(\gamma,b)}(m)$ and $q(d,r) = \sup_{k \ge 1} \frac{\Lambda_r(k)}{\Lambda_r((d+1)k)}$. Then,
\[
\Lambda_r(m) \le  \omega_m(\cN^r)\lesssim q(d,r) \Lambda_r(m).
\]
\end{theorem}
\begin{proof}
Let $\sigma_{w,b}(x)=\sigma(w^Tx+b)$. Then, for any  $\phi_1,\dots,\phi_m$,
\begin{align*}
\sup_{f\in \cN^r}   \inf_{c_1,\dots,c_{m}\in\RR}\| f - \sum_{j=1}^{m} c_j \phi_j\|_{\tau_{d-1}}^2&\geq\sup_{\gamma+|b|\leq r}\EE_{v\sim\tau_{d-1}}\inf_{c_1,\dots,c_{m}\in\RR}\| \sigma_{\gamma v,b} - \sum_{j=1}^{m} c_j \phi_j\|_{\tau_{d-1}}^2\\
&\geq \sup_{\gamma+|b|\leq r} \Lambda^{(\gamma,b)}(m) =  \Lambda_r(m),
\end{align*}
where the second inequality follows from Lemma \ref{pro: gen-feature}. Hence, the lower bound is proved. Let 
$
(c_1(w,b),\dots, c_m(w,b))=\argmin_{c_1,\dots,c_m} \|\sigma_{w,b}-\sum_{j=1}^m c_j \phi_j\|_{\tau_{d-1}}^2.
$
Then by Proposition \ref{pro: smooth-pointwise} with taking $L(m)=\Lambda_r(m)$, we have
$$
\|\sigma_{w,b}-\sum_{j=1}^m c_j(w,b) \phi_j\|_{\tau_{d-1}}^2\lesssim q(d,r)\Lambda_r(m).
$$
For any $f\in \cN^r$ and $\varepsilon>0$, there exist $\{(a_i, w_i, b_i)\}_{i}$ such that $\sum_{i}|a_i|\leq 1, \max_{i} (\|w_i\|_2+|b_i|)\leq r$ and
\[
	\|f - \sum_{i} a_i \sigma_{w_i,b_i}\|_{\tau_{d-1}}\leq \varepsilon.
\]
Let $\bar{c}_j = \sum_i a_i c_j(w_i,b_i)$. Then, 
\begin{align*}
\|f - \sum_{j=1}^m \bar{c}_j &\phi_j\|_{\tau_{d-1}} \leq \varepsilon + \Big\|\sum_i a_i \sigma_{w_i,b_i} - \sum_{j=1}^m \sum_i a_i c_j(w_i,b_i) \phi_j\big]\Big\|_{\tau_{d-1}} \\
&\leq \varepsilon+\sum_{i} |a_i|\|\sigma_{w_i,b_i}-\sum_{j=1}^m c_j(w_i,b_i)\phi_j\|_{\tau_{d-1}}\\
&\leq \varepsilon +\sum_{i}|a_i| \sqrt{q(d,r)\Lambda_r(m)}\leq \varepsilon+ \sqrt{q(d,r)\Lambda_r(m)}.
\end{align*} 
Taking $\varepsilon\to 0$,  we complete the proof.
\end{proof}

In the proof, the key ingredient is the uniform approximability of single neurons shown in Proposition \ref{pro: smooth-pointwise}. It is implied that $\Lambda_r(m)$ provide a tight bounds of the Kolmogorov width $\omega_m(\cN^r)$.  When $\Lambda_r(m)\sim m^{-s}$,  $q(d,r)\sim d^s$. In particular, when $\Lambda_r(m)=C_d m^{-\beta/d}$, we have $q(d,r)\leq d^{\beta/d}=O(1)$. This means that when $\Lambda_r(m)$ exhibits the CoD,  $\omega_m(\cN^r)\sim \Lambda_r(m)$, i.e., the spectral decay provides an exact description of the Kolmogorov width. Next, we study how the decay rate  is affected by  the norms of inner-layer widths and the smoothness of activation functions.

\subsection{Influence of the norms of inner-Layer weights }
 For ReLU , $\cN^r$ (up to a rescaling) are obviously the same for different $r$'s because of the homogeneity of $\sigma$. In particular, Theorem \ref{thm: gen-1} implies 
\begin{equation}
    \omega_m(\cN^r)\geq \frac{C(1/d)r^{2}}{m^{3/(d-1)}},
\end{equation}
where $C(1/d)$ depends on $1/d$ polynomially. Hence,  $\omega_m(\cN^r)$ exhibits the CoD for the ReLU activation function and the decay rate is independent of the value $r$. 
 However, for general activation functions, restricting $r$ may affects the decay rate.

 \begin{theorem}\label{thm: 2lnn-smooth-ub}
 Suppose that $\sigma\in C^\infty(\RR)$ satisfies Assumption \ref{assumption: smooth-activation} and $r=1$. Then, $
    \omega_m(\cN^1) \lesssim d/m.
$
The equality is reached by choosing the spherical harmonics as the fixed features.
 \end{theorem}

This theorem follows from a  simple combination of Proposition \ref{pro: smooth-lower-bound} and Proposition \ref{proposition: 2lnn-complete-characterization}.  For the specific arctangent activation, we have a fine-grained characterization as follows.
 \begin{proposition}\label{pro: 2lnn-smooth-ub-2}
 Assume $\sigma(t)=\arctan(t)$. We have 
\begin{equation}\label{eqn: 2lnn-arctan-upper}
    \omega_{m}(\cN^r) \lesssim \frac{d^4r^2}{m^{\min(0.5,r^{-2})}}.
\end{equation}  
The equality is reached by choosing the spherical harmonics as the fixed features.
 \end{proposition}
The key  idea is to estimate the integral \eqref{eqn: smooth-eta_k} in the Fourier domain using the Parseval's theorem. By using the explicit formula of the Fourier transform of $\sigma^{(k)}$ for the arctangent function, we show that the eigenvalue can  be expressed analytically using Gaussian hypergeometric functions. Then the integral representation of Gaussian hypergeometric functions is used for the estimation. The proof is quite technical and  deferred to Appendix \ref{sec: arctan-upper-bound}.

Theorem \ref{thm: 2lnn-smooth-ub} and Proposition \ref{pro: 2lnn-smooth-ub-2} imply that two-layer neural networks  have no clear separation from linear methods when the activation is smooth and the norms of inner-layer weights are bounded. Specifically, in this case, two-layer neural networks behave like polynomials in terms of approximation power. This is quite different from the ReLU case, where the separation of two type of methods is independent of the inner-layer weight norms.

Proposition \ref{pro: 2lnn-smooth-ub-2} implies that the error rate decreases with $r$ but independent of $d$ if $r=O(1)$. We conjecture that similar results hold for general smooth activation functions and some numerical supports are provided in Figure~\ref{fig: r-influence}. 
Specifically, we examine four activation functions including two sigmoid-like activations: Arctan and Sigmoid, and two ReLU-like activations: SiLU and softplus. 
According to Proposition \ref{proposition: 2lnn-complete-characterization}, $\Lambda_r$ is a good proxy of the Kolmogorov width.  The eigenvalues are numerically computed using Eq.~\eqref{eqn: eigen-sv}. In experiments, we find that $\Lambda_r=\Lambda^{(r,0)}$ for all the activation functions  examined.  Figure \ref{fig: r-influence} shows that  for all the cases, the rate is independent of $d$ for a fixed $r$, and  decreases with $r$ for a fixed $d$. This is consistent with Eq.~\eqref{eqn: 2lnn-arctan-upper}, which is only proved for the arctangent activation function.

\begin{figure}[!h]
\begin{subfigure}{0.5\textwidth}
    \includegraphics[width=0.49\textwidth]{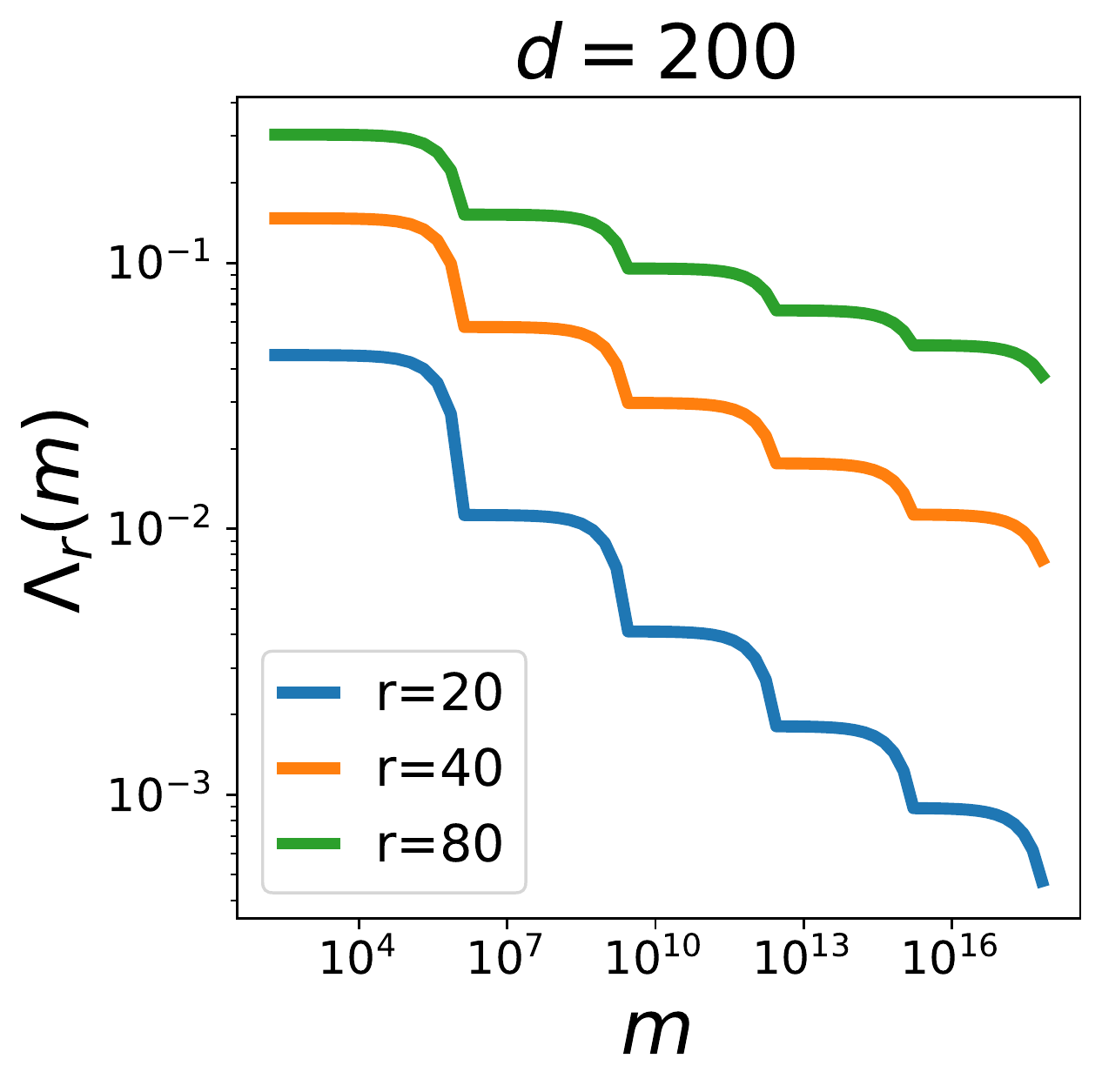}\hspace*{-1mm}
    \includegraphics[width=0.49\textwidth]{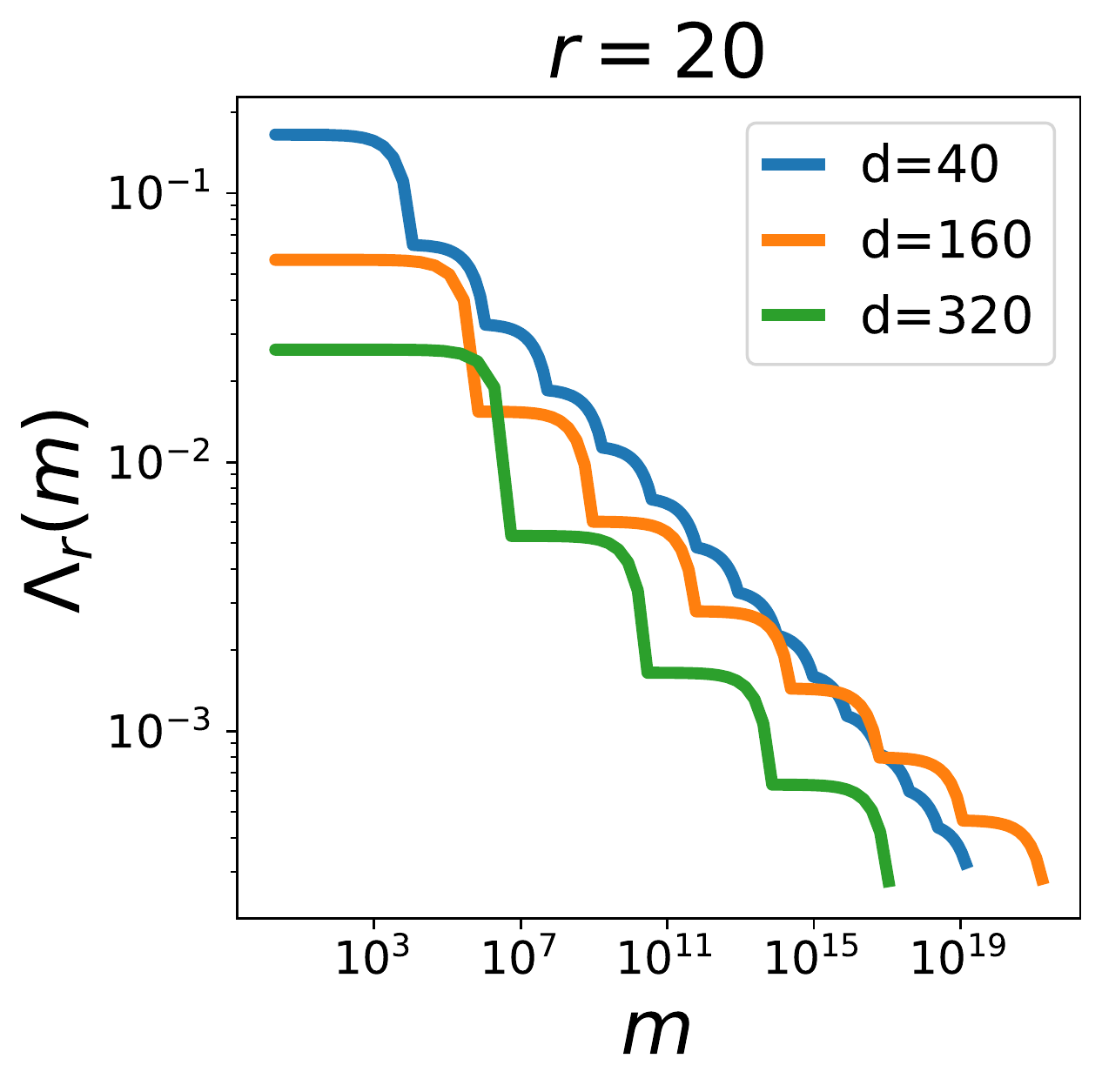}
    \vspace*{-3mm}
    \caption{\small Arctan.}
\end{subfigure}
\hspace*{-2mm}
\begin{subfigure}{0.5\textwidth}
    \includegraphics[width=0.49\textwidth]{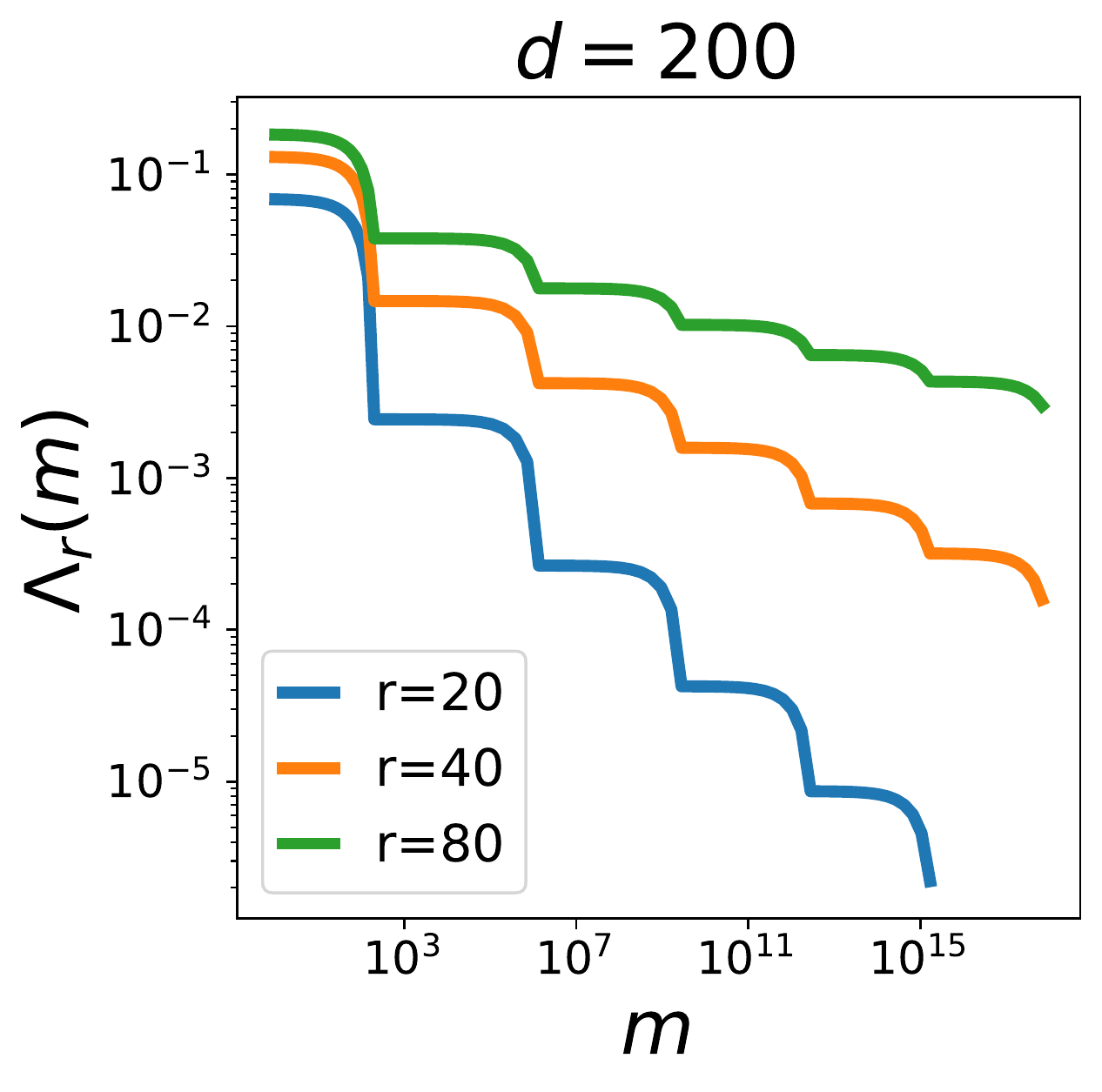}\hspace*{-1mm}
    \includegraphics[width=0.49\textwidth]{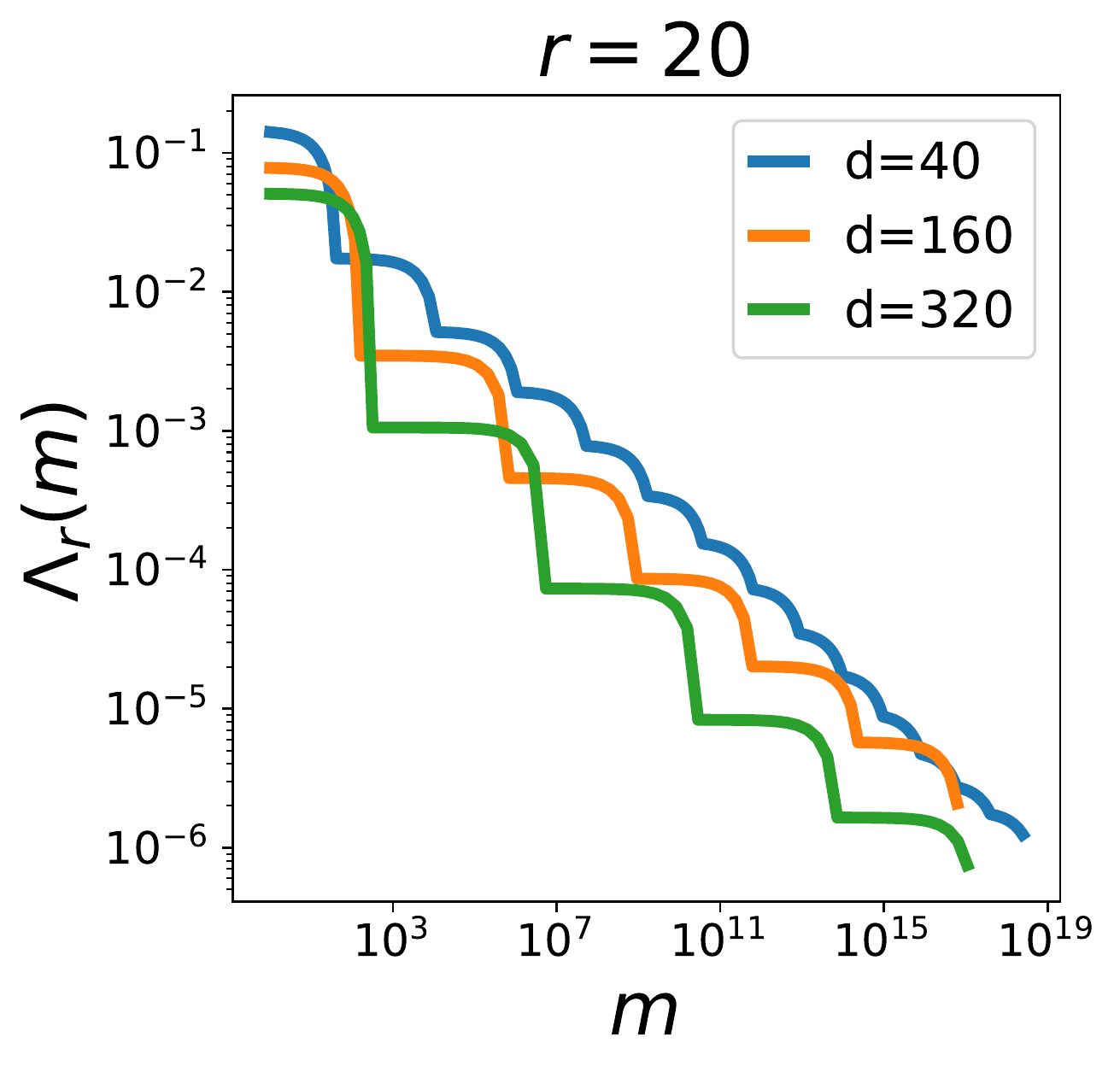}
    \vspace*{-3mm}
    \caption{\small Sigmoid.}
\end{subfigure}

\begin{subfigure}{0.5\textwidth}
    \includegraphics[width=0.49\textwidth]{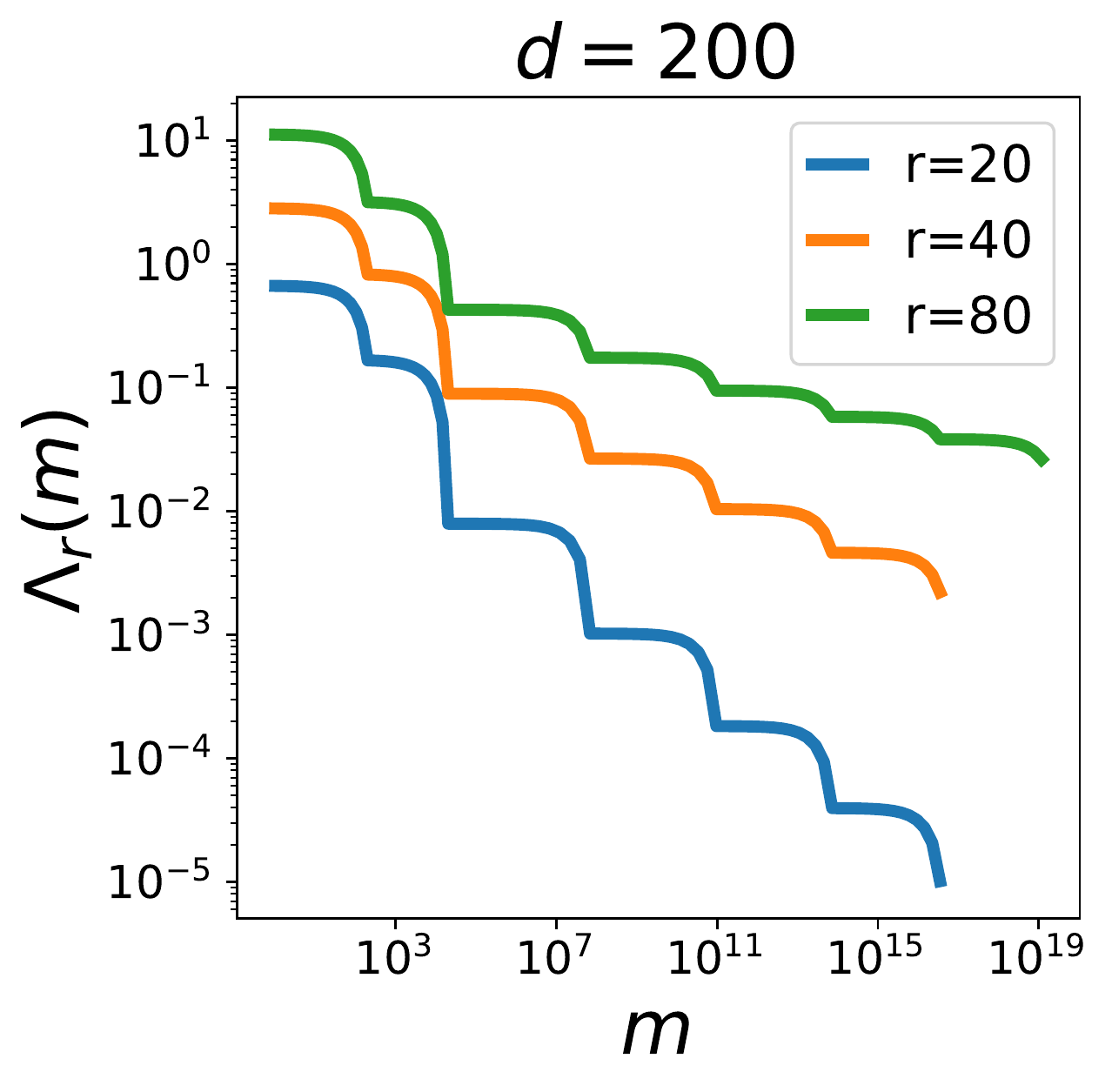}\hspace*{-1mm}
    \includegraphics[width=0.49\textwidth]{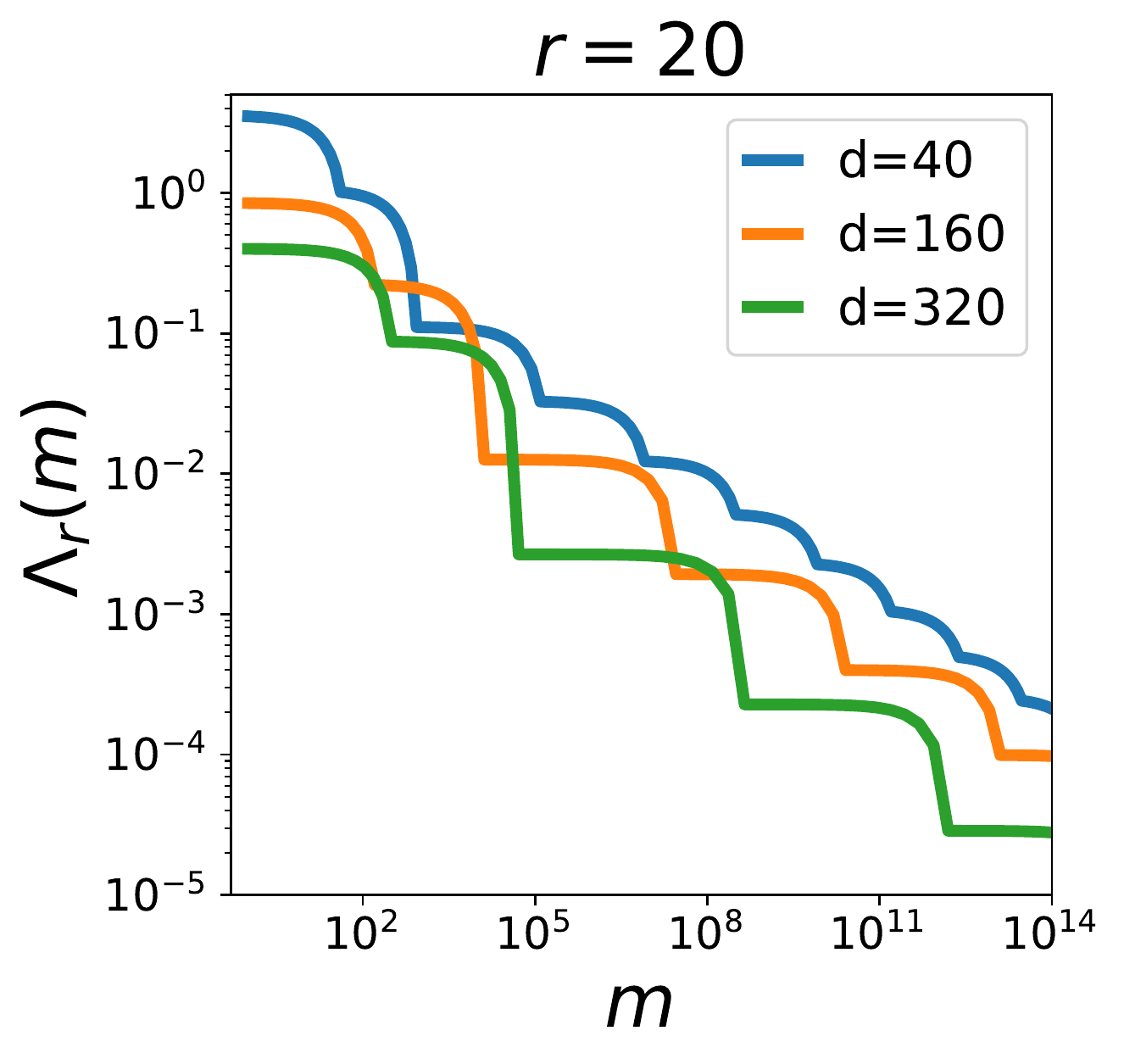}
    \vspace*{-3mm}
    \caption{\small SiLU.}
\end{subfigure}
\hspace*{-2mm}
\begin{subfigure}{0.5\textwidth}
    \includegraphics[width=0.49\textwidth]{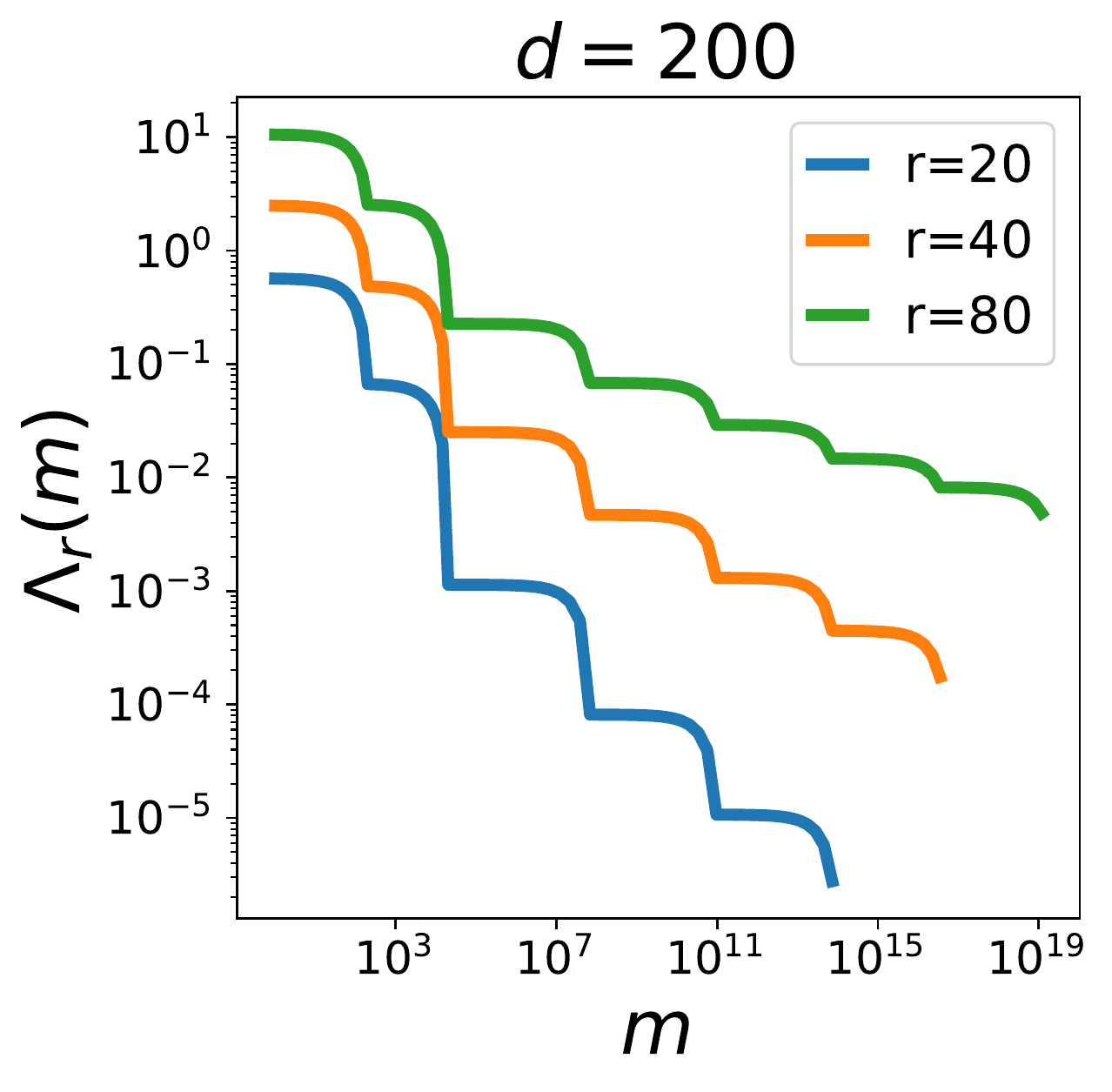}\hspace*{-1mm}
    \includegraphics[width=0.49\textwidth]{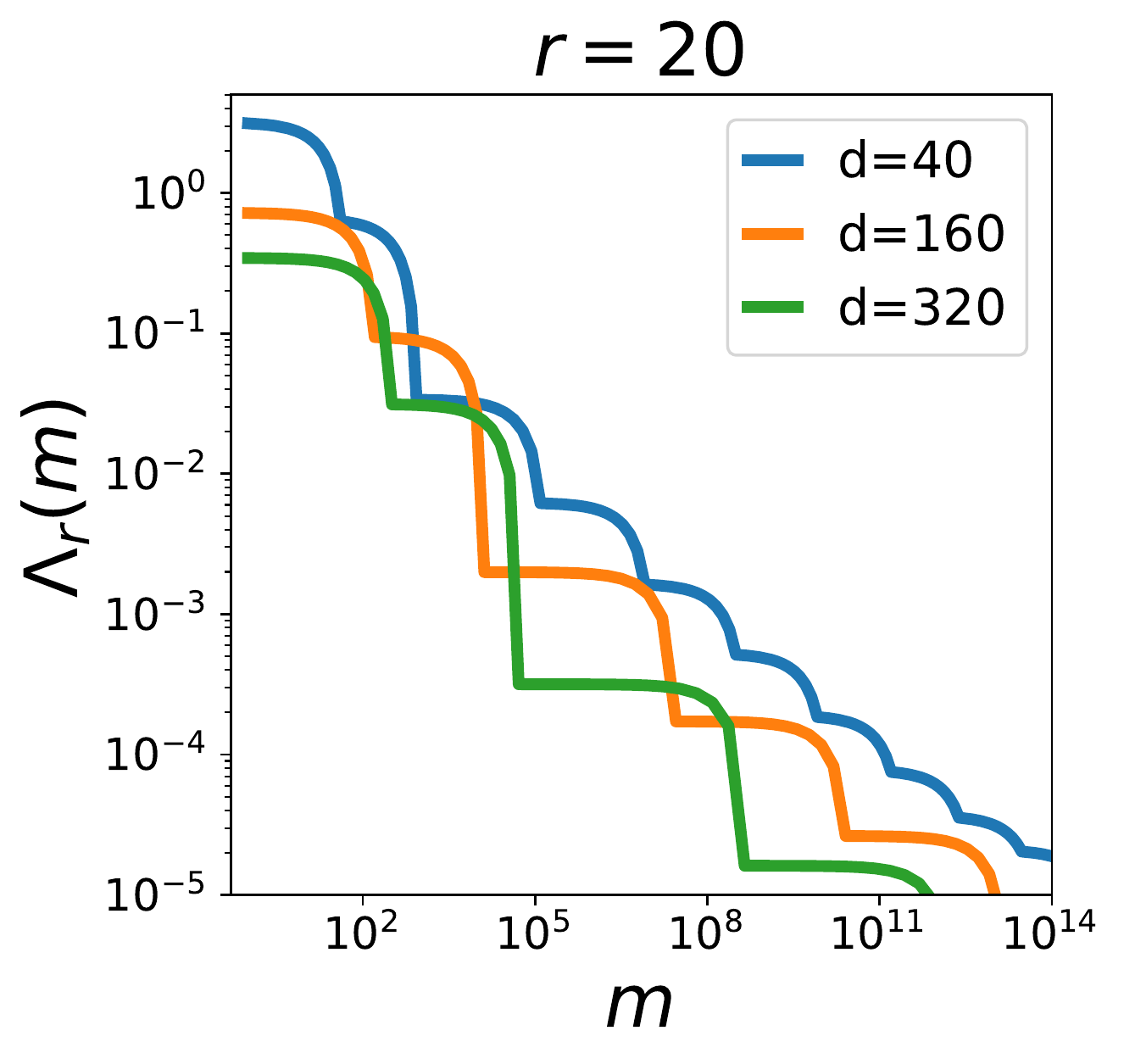}
    \vspace*{-3mm}
    \caption{\small Softplus.}
\end{subfigure}
\vspace*{-2mm}
\caption{How the decay of $\Lambda_r(m)$, thereby the linear approximability, changes with $r$ for fixed $d$ (left), and changes with $d$ for fixed $r$ (right). Two sigmoid-like and two ReLU-like smooth activation functions are examined.
}
\label{fig: r-influence}
\end{figure}

Note that a  result similar to Proposition \ref{pro: 2lnn-smooth-ub-2}  has be provided in \cite{livni2014computational} for the sigmoid activation function. Ours differs from \cite{livni2014computational} in two aspects. First, we show that the same observation holds for more general smooth activation functions by clear numerical evidences. In particular, Proposition \ref{proposition: 2lnn-complete-characterization} combined with Lemma \ref{lemma: eigen} provides us an easy way to numerically compute upper bounds. Second, we can  establish a hardness result given below. These improvements are benefiting from our spectral-based analysis.
% , and it is impossible to obtain these results  using the techniques in \cite{livni2014computational}.

Next we then show that when $r$ is polynomially large with respect to $d$, even for smooth activation function, $\omega_m(\mathcal{N}^r)$ exhibits the CoD.

% \vspace*{-2cm}
\begin{assumption}\label{assum: smooth}
Suppose that the activation function $\sigma$ satisfies either $|\sigma(rt)-\step(t)|\lesssim  (1+r|t|)^{-\beta}$ or $|\sigma(rt)-r\relu(t)|\leq (1+r|t|)^{-\beta}$ for any $t$ and some constant $\beta>0$.
\end{assumption}
This assumption is satisfied by all the commonly-used activation functions.

\begin{theorem}\label{thm: 2lnn-arctan-lower-bound}
Suppose  that $\sigma$ satisfies Assumption \ref{assum: smooth} and $m\leq 2^d$. Then, there exists constants $C_1(\beta), C_2>0$  such that if  $r\geq  d^{C_1(\beta)}$, we have 
$
    \omega_{m}(\cN^r)\gtrsim d^{-C_2}.
$
\end{theorem}
\begin{proof}
We only present the proof for the sigmoid-like activation functions. The proof for ReLU-like ones  is similar can be found in Appendix \ref{sec: missing-proof-thm3}.
For any $v\in\SS^{d-1}$,
\begin{align*}
    \|\sigma(rv^T\cdot)-&\step(v^T\cdot)\|_{\tau_{d-1}}^2 = \frac{1}{B(\frac{1}{2}, \frac{d-1}{2})}\int_{-1}^1  (\sigma(rt)-\step(t))^2 (1-t^2)^{(d-3)/2}\dd t,
\end{align*}
where we use the fact that the density function of $v^Tx$ is $p(t)=\frac{(1-t^2)^{\frac{d-3}{2}}}{B(\frac 1 2,\frac{d-1}{2})}$. Then, we have 

\begin{align*}
    \|\sigma(rv^T\cdot)-&\step(v^T\cdot)\|_{\tau_{d-1}}^2 \leq \frac{2}{B(\frac{1}{2}, \frac{d-1}{2})}\left(\int_0^{\delta}(1-t^2)^{(d-3)/2}\dd t + \int_{\delta}^1  \frac{1}{(1+rt)^{2\beta}}(1-t^2)^{(d-3)/2}\dd t\right)\\
    &\lesssim d^{1/2}\delta +\left(\frac{1}{r \delta}\right)^{2\beta},
\end{align*}
Let $h(t)=\sqrt{\frac{d}{r^2}} t + t^{-2\beta}$. Then, $h'(t)=\sqrt{\frac{d}{r^2}}  - 2\beta t^{-2\beta-1}=0$ leads to $\bar{t}=(2\beta\sqrt{r^2/d})^{1/(1+2\beta)}$. Hence,
\[
\|\sigma(rv^T\cdot)-\step(v^T\cdot)\|_{\tau_{d-1}}^2\lesssim \inf_{\delta\in [0,1]}\Big(d^{1/2}\delta +\big(\frac{1}{r \delta}\big)^{2\beta}\Big) = C(\beta) \left(\frac{d}{r^2}\right)^{\frac{\beta}{2(1+2\beta)}}.
\]
By the triangle inequality,
\begin{align*}
\notag \|\sigma(rv^T\cdot)-\sum_{j=1}^m c_j \phi_j\|^2_{\tau_{d-1}}&\gtrsim \|\step(v^T\cdot)-\sum_{j=1}^m c_j \phi_j\|^2_{\tau_{d-1}} - \|\step(v^T\cdot)-\sigma(rv^T\cdot)\|^2_{\tau_{d-1}}\\
&\gtrsim \|\step(v^T\cdot)-\sum_{j=1}^m c_j \phi_j\|^2_{\tau_{d-1}} - C(\beta) \left(\frac{d}{r^2}\right)^{\frac{\beta}{2(1+2\beta)}}.
\end{align*}
Using Theorem \ref{thm: gen-1}, there exist a constant $C_2>0$ such that 
\begin{align*}
\notag \EE_{v\sim\tau_{d-1}}\inf_{c_1,\dots,c_m}\|\sigma(rv^T\cdot)-\sum_{j=1}^m c_j \phi_j\|^2_{\tau_{d-1}}
&\gtrsim \frac{1}{d^{C_2}m^{1/(d-1)}} - C(\beta) \left(\frac{d}{r^2}\right)^{\frac{\beta}{2(1+2\beta)}}.
\end{align*}
Let $C(\beta) \left(\frac{d}{r^2}\right)^{\frac{\beta}{2(1+2\beta)}}\lesssim d^{-C_2}$. This leads to $r\geq C(\beta)^{\frac{(1+2\beta)}{\beta}} d^{\frac{1}{2}+\frac{C_2(1+2\beta)}{\beta}}$. Hence,
there exist $C_1(\beta)>0$ such that if $m\leq 2^d$ and $r\geq d^{C_1(\beta)}$, we must have 
$$
\EE_{v\sim\tau_{d-1}}\inf_{c_1,\dots,c_m}\|\sigma(rv^T\cdot)-\sum_{j=1}^m c_j \phi_j\|^2_{\tau_{d-1}} \gtrsim \frac{1}{d^{C_2}}.
$$ 
\end{proof}

In particular, for the arctangent activation function, we have the following refined result, whose proof is deferred to Appendix \ref{sec: proof-lower-bound-actan}.
\begin{theorem}\label{thm: low-bound-actan}
Suppose that $\sigma(t)=\arctan(t)$ and $r=d^\alpha$. If $\alpha>\frac{1}{2}$, 
there exists constants $C_1, C_2(\alpha)>0$ such that if $m\leq d^{-C_1}2^{C_2(\alpha)d^{2\alpha-1}}$,
$\omega_m(\cN^r)\gtrsim d^{-3}$.
\end{theorem}
Let $r=d^\alpha$. Theorem \ref{thm: low-bound-actan} improves Theorem \ref{thm: 2lnn-arctan-lower-bound} by showing that $\alpha>1/2$ is sufficient to establish the CoD-type lower bound. We conjecture that the same results also hold for more general activation functions. We leave this to future work. These theorems imply that $\omega_m(\cN^r)$ also exhibits the CoD for sigmoid-like and ReLU-like smooth activations as long as the norms of inner-layer weights are polynomial in $d$.  

This result is related to \cite[Theorem 5]{livni2014computational}, which shows that the time complexity of learning $\cN^{\poly(d)}$ is exponential in $d$.  However, \cite[Theorem 5]{livni2014computational}  relies on  cryptographic assumptions. These assumptions mean that  some standard hard problems cannot be learned in polynomial time; otherwise modern cryptosystems can be broken in polynomial time. The theoretical justification of these assumptions remains an open problem, although they are believed to be true.  
By contrast, our result is unconditional and does not rely on any hardness assumption. Note that two results are generally not comparable, since ours is for the approximation complexity whereas \cite[Theorem 5]{livni2014computational} is for training complexity.

Lastly, we mention that  some previous studies of two-layer neural networks constrain $r$ to be finite (see, e.g., \citealt{chen2020dynamical}). Our result suggests that one should be careful about the value of $r$, otherwise the result may be not able to distinguish neural networks from linear methods.

\section{Conclusion}
In this paper, we provide a systematic study of the separation between two-layer neural networks and linear methods in terms of approximation functions in high dimension. 
To this end, we develop a spectral-based approach, which reduces the problem to computing the eigenvalues of an associated kernel. Our approach allows obtaining upper bounds, lower bounds, and identifying explicit hard examples simultaneously.  We extend and improve the previous separation results for the sigmoidal and ReLU$^\alpha$ activation functions to general nonsmooth activation functions. We also find that for smooth activation functions, whether the separation exists or not crucially depends on the inner-layer weight norms.

Technically speaking, our spectral-based approach provides a way to accurately compute the Kolmogorov width of two-layer neural networks. This approach should be also applicable to analyze other properties that are related to the Kolmogorov width, e.g., the metric entropy. We leave this to future work.

\appendix

\section{Missing proofs of Section \ref{sec: single-neuron} }\label{eqn: single-neuron-proof}
Here, we present the missing lengthy proofs of Section \ref{sec: single-neuron}.

\subsection{Proof of Proposition \ref{pro: smooth-pointwise}}
\label{sec: single-neuron-smooth-pointwise}
To prove Proposition \ref{pro: smooth-pointwise}, we first need the following lemma.
\begin{lemma}\label{lemma: point-wise-nonsmooth}
For any activation function $\sigma$ and any $v \in \bS^{d-1}$,
\begin{equation*}
    \inf_{\{c_{i,j}\}_{1\le i \le k, 1\le j \le N(d,i)}}\|\sigma_v - \sum_{i = 0}^k \sum_{j=1}^{N(d,i)}c_{i,j}Y_{i,j}\|_{\tau_{d-1}}^2 = \sum_{i = k+1}^{\infty}N(d,i)\mu_i.
\end{equation*}
\end{lemma}
\begin{proof}
Recalling that $\{Y_{i,j}\}_{0\le i \le k,1\le j \le N(d,i)}$ is orthonormal in $L^2(\tau_{d-1})$, we have
\begin{equation*}
     \inf_{\{c_{i,j}\}_{0\le i \le k, 1\le j \le N(d,i)}}\|\sigma_v - \sum_{i = 0}^k \sum_{j=1}^{N(d,i)}c_{i,j}Y_{i,j}\|_{\tau_{d-1}}^2 = \|\sigma_v\|_{\tau_{d-1}}^2 - \sum_{i=0}^k\sum_{j=1}^{N(d,i)}\langle Y_{i,j},\sigma_v\rangle_{\tau_{d-1}}^2.
\end{equation*}
Hence, using Eq.~\eqref{eqn: sum_of_Y}
\begin{align*}
    &\|\sigma_v\|_{\tau_{d-1}}^2 - \sum_{i=0}^k\sum_{j=1}^{N(d,i)}\langle Y_{i,j},\sigma_v\rangle_{\tau_{d-1}}^2\\ 
    =& \int_{\bS^{d-1}}|\sigma(v\transpose x)|^2\dd\tau_{d-1}(x) - \sum_{i=0}^k\int_{\bS^{d-1}}\int_{\bS^{d-1}}\sigma(v\transpose x)\sigma(v\transpose x')\sum_{j=1}^{N(d,i)}Y_{i,j}(x)Y_{i,j}(x')\dd\tau_{d-1}(x)\dd\tau_{d-1}(x') \\
    =&\int_{\bS^{d-1}}|\sigma(v\transpose x)|^2\dd\tau_{d-1}(x) - \sum_{i=0}^kN(d,i)\int_{\bS^{d-1}}\int_{\bS^{d-1}}\sigma(v\transpose x)\sigma(v\transpose x')P_i(x\transpose x')\dd\tau_{d-1}(x)\dd\tau_{d-1}(x').
\end{align*}
By the rotational invariance of $\tau_{d-1}$, we know that the above equation is constant for all $v \in \bS^{d-1}$. Hence,
\begin{align*}
     &\inf_{\{c_{i,j}\}_{0\le i \le k, 1\le j \le N(d,i)}}\|\sigma_v - \sum_{i = 0}^k \sum_{j=1}^{N(d,i)}c_{i,j}Y_{i,j}\|_{\tau_{d-1}}^2 \\=&\int_{\bS^{d-1}}[\int_{\bS^{d-1}}|\sigma(v\transpose x)|^2\dd\tau_{d-1}(v)]\dd\tau_{d-1}(x) \\
     -&\sum_{i=0}^kN(d,i)\int_{\bS^{d-1}}\int_{\bS^{d-1}}[\int_{\bS^{d-1}}\sigma(v\transpose x)\sigma(v\transpose x')\dd\tau_{d-1}(v)]P_i(x\transpose x')\dd\tau_{d-1}(x)\dd\tau_{d-1}(x')\\
     =& \int_{\bS^{d-1}}\kappa(x\transpose x)\dd\tau_{d-1}(x) - \sum_{i=0}^k N(d,i)\int_{\bS^{d-1}}\int_{\bS^{d-1}}\kappa(x\transpose x')P_i(x\transpose x')\dd\tau_{d-1}(x)\dd\tau_{d-1}(x').
\end{align*}
Combining the last equation and Eq.~\eqref{eqn: sum_of_Y} and Eq.~\eqref{eqn: def_mu}, we obtain that
\begin{align*}
     &\inf_{\{c_{i,j}\}_{0\le i \le k, 1\le j \le N(d,i)}}\|\sigma_v - \sum_{i = 0}^k \sum_{j=1}^{N(d,i)}c_{i,j}Y_{i,j}\|_{\tau_{d-1}}^2 \\
     =&\sum_{i=0}^{\infty}\sum_{j=1}^{N(d,i)}\mu_i - \sum_{i=0}^{k}\sum_{j=1}^{N(d,i)}\mu_i[\int|Y_{i,j}(x)|^2\dd \tau_{d-1}(x)]^2 =\sum_{i=k+1}^{\infty}N(d,i)\mu_i.
\end{align*}
\end{proof}

\paragraph*{Proof of Proposition \ref{pro: smooth-pointwise}}
Let $m_k=\sum_{i=0}^k N(d,i)$. By Lemma \ref{lemma: point-wise-nonsmooth} and Proposition \ref{pro: smooth-lower-bound}, we have 
\begin{equation*}
    \inf_{c_1,\dots,c_{m_k}}\|\sigma_v -  \sum_{j=1}^{m_k}c_j \phi_j\|_{\tau_{d-1}}^2 = \Lambda(m_k) \le L(m_k).
\end{equation*}
For any $m$, assume $m\in [m_{k-1}+1, m_{k}]$. Then, 
\begin{align}\label{eqn: mk-ratio}
   \inf_{c_1,\dots,c_{m}}\|\sigma_v -  \sum_{j=1}^{m}c_j \phi_j\|_{\tau_{d-1}}^2 &\leq \inf_{c_1,\dots,c_{m_{k-1}}}\|\sigma_v -  \sum_{j=1}^{m_{k-1}}c_j \phi_j\|_{\tau_{d-1}}^2 \le L(m_{k-1}) \le \frac{L(m_{k-1})}{L(m_k)}L(m)
\end{align}
By \eqref{eqn: x2} and \eqref{eqn: N-asym}, we have
\begin{align*}
m_{k} &= \frac{\Gamma(k+d)}{\Gamma(d)\Gamma(k+1)} + \frac{\Gamma(k+d-1)}{\Gamma(d)\Gamma(k)}, 
\end{align*}
when $k \ge 1$ and $m_0 = 0$
Then,
\begin{align*}
    \frac{m_{k+1}}{m_{k}} &\le \max\{\frac{\Gamma(k+d+1)}{\Gamma(d)\Gamma(k+2)}\Big/\frac{\Gamma(k+d)}{\Gamma(d)\Gamma(k+1)},\frac{\Gamma(k+d)}{\Gamma(d)\Gamma(k+1)}\Big/\frac{\Gamma(k+d-1)}{\Gamma(d)\Gamma(k)}\} \\
    &= \max\{\frac{k+d}{k+1},\frac{k+d-1}{k}\}\le d+1,
\end{align*}
and 
\begin{equation*}
    \frac{m_1}{m_0} = d + 1.
\end{equation*}
Therefore, 
\[
    \frac{m_{k}}{m_{k-1}}\le d+1,
\]
which means that
\begin{equation*}
    \frac{L(m_{k-1})}{L(m_k)} \le \frac{L(m_{k-1})}{L((d+1)m_{k-1})}\le q(d,L).
\end{equation*}
Plugging it into \eqref{eqn: mk-ratio},   we complete the proof.

\subsection{Proof of Proposition \ref{pro: non-smooth-1}}
\label{sec: single-non-smooth}

We only give the proof when $d \ge 3$. The simple case $d=2$ can also be proven using similar argument.
According to Appendix D.2 of \cite{bach2017breaking}~\footnote{Note that the one provided in \cite{bach2017breaking} is not correct due to the miscalculation of $\omega_{d-2}/\omega_{d-1}$.},   we have 
\begin{align*}\label{eqn: non-smooth-eigenvalue} 
\eta_k = \begin{cases}
C_1(\frac 1 d,\alpha) & \text{if } k\leq \alpha, \\
0  & \text{if } k\geq \alpha+1, \text{and } k\equiv \alpha\pmod{2}\\
\frac{\Gamma(\alpha+1)}{\sqrt{2\pi}\,2^k}\frac{\Gamma(d/2)\Gamma(k-\alpha)}{\Gamma(\frac{k-\alpha+1}{2})\Gamma(\frac{k+d+\alpha}{2})} & \text{otherwise},
\end{cases}
\end{align*}
where $C_1(\frac 1 d,\alpha)$ depends on $\frac 1 d $ polynomially.
By Stirling formula, $\Gamma(t)\sim \sqrt{2\pi} t^{t-1/2}e^{-t+1}$ for any $t\geq 1$. Then, up to a constant only depending on $\alpha$, we have for $k\geq \alpha+1$ and $k\equiv (\alpha+1)\pmod{2}$, 
\begin{equation*}\label{eqn: x1}
    \mu_k \sim   d^{d-1} k^{k-\alpha-1}(k+d)^{-k-d-\alpha + 1}.
\end{equation*}
Note that 
\begin{equation*}\label{eqn: x2}
N(d,k) =  \frac{\Gamma(k+d)}{\Gamma(d)\Gamma(k+1)} - \frac{\Gamma(k+d-2)}{\Gamma(d)\Gamma(k-1)}
\end{equation*}
where up to a constant, we have 
\begin{equation}\label{eqn: N-asym}
    \frac{\Gamma(k+d)}{\Gamma(d)\Gamma(k+1)}\sim (k+d)^{k+d-\frac{1}{2}}k^{-\frac{1}{2}-k}d^{\frac{1}{2}-d}
\end{equation}

WLOG, assume that $\alpha$ is odd. 
By definition,  
\begin{align}\label{eqn: m-k-condition}
    \lambda_m=\mu_{2k},\notag \text{for } 
     m \in \Big[&\sum_{i=0}^{\alpha} N(d,i)+1 +\frac{\Gamma(2k+d-2)}{\Gamma(d)\Gamma(2k-1)}-\frac{\Gamma(\alpha + d - 1)}{\Gamma(d)\Gamma(\alpha)},\notag\\
     &\sum_{i=0}^{\alpha} N(d,i)+ \frac{\Gamma(2k+d)}{\Gamma(d)\Gamma(2k+1)}-\frac{\Gamma(\alpha + d - 1)}{\Gamma(d)\Gamma(\alpha)}\Big]
\end{align}
By \eqref{eqn: N-asym}, when 
\begin{equation}\label{eqn: m-condition}
    m \ge 1 + \sum_{i=0}^{\alpha}N(d,i) = C(\frac{1}{d},\alpha),
\end{equation}
and there exists $k$ such that \eqref{eqn: m-k-condition} is satisfied, 
we have
\[
m \gtrsim (2k+d)^{2k+d - \frac{5}{2}}(2k)^{\frac{3}{2}-2k}d^{\frac{1}{2}-d},
\]
and, 
\[
    \lambda_m = \mu_{2k} \sim (2k+d)^{-2k-d-\alpha+ 1}(2k)^{2k-\alpha-1}d^d,
\]
which means
\begin{equation}\label{eqn: eigenvalues}
    \lambda_m m^{\frac{d+2\alpha}{d-1}} \ge C(\frac{1}{d},\alpha) \Big(\frac{2k+d}{2k}\Big)^{\frac{2\alpha+1}{d-1}(2k+d-\frac{5}{2}) - \alpha - \frac{3}{2}}\ge C(\frac{1}{d},\alpha).
\end{equation}
Noticing that when \eqref{eqn: m-condition} is not satisfied, the above inequality still holds, we have 
\[
\Lambda(m)=\sum_{j=m+1}^\infty \lambda_j \geq C(\frac{1}{d}, \alpha) m^{-\frac{2\alpha+1}{d-1}}.
\]
In particular, when $\alpha=0$, $C(1/d,0)=1/d$.

\subsection{Proof of Proposition \ref{pro: smooth-lower-bound}}
\label{sec: proof-single-neuron-smooth-trace}

First, by \eqref{eqn: x1} and \eqref{eqn: x2},
$$
\sum_{j=0}^k N(d,j)\lesssim \frac{2\Gamma(k+d)}{\Gamma(d)\Gamma(k+1)}\sim (k+d)^{k+d-1/2}k^{-1/2-k} d^{-d+1/2}.
$$
Second, plugging $B_k\lesssim k!$ into Lemma \ref{lemma: smooth-activation} and using the Stirling formula, we have  
\[
    \mu_k \lesssim  \frac{k^{2k+1}e^{-2k}}{2^{2k}}\left(\frac{(d/2)^{d/2-1/2}e^{-d/2+1}}{(k+d/2)^{k+d/2-1/2}e^{-k-d/2+1}}\right)^2\sim k^{2k+1}d^{d-1}(2k+d)^{-2k-d+1}.
\]
Let $\bar{\lambda}_m = \mu_k$ if $m \in [\sum_{j=0}^{k-1}N(d,j) +1, \sum_{j=0}^kN(d,j)]$, then
\begin{equation*}
    \Lambda(m)\le \sum_{j= m+1}^{\infty}\bar{\lambda}_j
\end{equation*}
because $\{\lambda_j\}_{j \ge 0}$ is the non-increasing rearrangement of $\{\bar{\lambda}_j\}_{j \ge 0}$.

By definition, if $\bar{\lambda}_m=\mu_k$, we must have $m\leq \sum_{j=0}^k N(d,j)$. Therefore, 
\begin{align*}
\bar{\lambda}_m m^2 &\lesssim \mu_k  (\sum_{j=0}^k N(d,j))^2 \leq d^{-d} (k+d)^{2k+2d-1} (2k+d)^{-2k-d+1}  \\
&= \left(\frac{k+d}{d}\right)^{d} \left(\frac{k+d}{2k+d}\right)^{2k+d-1}\leq \left(\frac{k+d}{d}\right)^{d} \left(\frac{k+d}{2k+d}\right)^{2k+d}.
\end{align*}
Let $k=t d$, then we have 
\begin{align}\label{eqn: lambda_m-2}
\bar{\lambda}_m m^2 \lesssim (t+1)^d \left(\frac{t+1}{2t+1}\right)^{(2t+1)d} =: h(t)^d,
\end{align}
where $h(t)=(t+1)^{2t+2}/(2t+1)^{2t+1}$. Hence, $\log h(t)=(2t+2)\log(t+1) - (2t+1)\log(2t+1)$.
\begin{align}
\frac{\dd\log h(t)}{\dd t} = 2\log(t+1) + 2 - 2 \log(2t+1) - 2 \leq 0,\, \forall\, t\geq 0.
\end{align}
Combing with $\log h(0)=0$, we have $\log h(t)\leq 0$, i.e., $h(t)\leq 1$, for any $t\geq 0$. Plugging it into \eqref{eqn: lambda_m-2} leads to $\bar{\lambda}_m m^2 \lesssim 1$. In other words, $\bar{\lambda}_m \lesssim 1/m^2$. Therefore, we have 
\[
\Lambda(m) \le \sum_{j=m+1}^\infty \bar{\lambda}_j \lesssim \frac{1}{m}.
\]

$\qed$

\section{Missing proofs of Section \ref{sec: two-layer-network}}

\subsection{Proof of Proposition \ref{pro: 2lnn-smooth-ub-2}}
\label{sec: arctan-upper-bound}

We show that for the arctangent activation function, the eigenvalues can be expressed analytically using the hypergeometric functions. 
\begin{lemma}\label{pro: arctan-analytic-expression}
Assume $\sigma(z)=\arctan(rz+b)$. Then, 
\begin{align}
|\eta_k|\le Q_{d,k}  \int_{0}^1 \left(\frac{r^2}{1+r^2t}\right)^{\frac{k}{2}} t^{\frac{k-1}{2}}(1-t)^{\frac{k+d-3}{2}} \dd t,
\end{align}
where $Q_{d,k}=\frac{\Gamma(k)\Gamma(\frac{d}{2})}{2^k\Gamma(\frac{k+1}{2})\Gamma(\frac{k+d-1}{2})}$. The equality is reached when $b=0$.
\end{lemma}
\begin{proof}
Recall  \eqref{eqn: Rodrigues},
\[
\eta_k = \frac{1}{2^k} \frac{\omega_{d-2}}{\omega_{d-1}}\frac{\Gamma((d-1)/2)}{\Gamma(k+(d-1)/2)} \int_{-1}^1 \sigma^{(k)}(t)\left(1-t^{2}\right)^{k+(d-3) / 2} \dd t.
\]
By Parseval's theorem, 
\begin{equation}\label{eqn: parseval}
    \eta_k = \frac{1}{2^k} \frac{\omega_{d-2}}{\omega_{d-1}}\frac{\Gamma((d-1)/2)}{\Gamma(k+(d-1)/2)} \frac{1}{2\pi} \int \cF(\sigma^{(k)})(\xi)\cF(\nu_{d,k})(\xi)\dd\xi,
\end{equation}
where $\nu_{k,d}(t)= (1-t^2)^{k + \frac{d-3}{2}}$ and $\cF(\cdot)$ denote the Fourier transform, i.e.,  $\cF(f)(\xi)=\int f(t)e^{-i\xi t}\dd t$.

 Notice that 
\begin{equation}\label{eqn: sigma-fourier}
\sigma'(z)=\frac{r}{1+(rz+b)^2}\qquad \cF(\sigma')(\xi)=\pi e^{i\frac{b}{r}\xi-|\xi|/r}\qquad \cF(\sigma^{(k)})(\xi) = (i\xi)^{k-1} \pi e^{i\frac{b}{r}\xi-|\xi|/r}.
\end{equation}
Also, using Eqn. (10.9.4) of \cite{olver2010nist}, we know that
\begin{equation}\label{eqn: nu-fourier}
    \mathcal{F}(\nu_{k,d})(\xi) = \sqrt{2\pi}2^{k + \frac{d-3}{2} }\Gamma\left(k + \frac{d-1}{2}\right)|\xi|^{-k - \frac{d-2}{2}} J_{k + \frac{d-2}{2}}(|\xi|),
\end{equation}
where $J_{k + \frac{d-2}{2}}$ is the Bessel function of the first kind with order $k +\frac{d-2}{2}$ (see, e.g., \cite[10.2.2]{olver2010nist}).
Plugging \eqref{eqn: sigma-fourier} and \eqref{eqn: nu-fourier} into \eqref{eqn: parseval} leads to 
\begin{align*}
   \eta_k &=2^{\frac{d-2}{2}}\Gamma\left(\frac{d}{2}\right) \frac{1}{2\pi}\int_{-\infty}^\infty |\xi|^{k-1}\pi e^{-|\xi|/r+i\frac{b}{r}\xi} |\xi|^{-k-\frac{d-2}{2}} J_{k+\frac{d-2}{2}}(|\xi|)\dd\xi\\
    &= 2^{\frac{d-2}{2}}\Gamma\left(\frac{d}{2}\right)\int_0^\infty \xi^{-d/2} J_{k+\frac{d-2}{2}}(\xi) e^{-\xi/r} i^{k-1}[e^{i\frac{b}{r}\xi} +e^{-i\frac{b}{r}\xi}] \dd\xi.
\end{align*}
By Eq.~(10.22.49) of \cite{olver2010nist}, the above integration can be expressed by using the Gaussian hypergeometric function $\bF$ (see \cite[Chapter 15]{olver2010nist}):
\begin{align}\label{eqn: eta-k-2}
\eta_k =\Gamma(k)\Gamma(\frac{d}{2})\left(\frac{r}{2}\right)^{k}\frac{i^{k-1}}{2}[\bF\left(\frac{k}{2},\frac{k+1}{2};k+\frac{d}{2};-\frac{r^2}{(1+bi)^2}\right) + \bF\left(\frac{k}{2},\frac{k+1}{2};k+\frac{d}{2};-\frac{r^2}{(1-bi)^2}\right)],
\end{align}
Notice that $\bF(p,q;u;z)$ has the integral representation \cite[Section 15.6]{olver2010nist} as follows:
\[
    \bF(p,q;u;z) = \frac{1}{\Gamma(q)\Gamma(u-q)}\int_0^1 \frac{t^{q-1}(1-t)^{u-q-1}}{(1-zt)^p}\dd t.
\]
Plugging it into \eqref{eqn: eta-k-2} gives us
\begin{align*}
|\eta_k|&\le\frac{\Gamma(k)\Gamma(\frac{d}{2})}{2^k\Gamma(\frac{k+1}{2})\Gamma(\frac{k+d-1}{2})}\int_0^1\left(\frac{r^2}{1+r^2t}\right)^{k/2} t^{\frac{k-1}{2}}(1-t)^{\frac{k+d-3}{2}} \dd t.
\end{align*}
Thus, we complete the proof.
\end{proof}

Let 
\begin{equation}\label{eqn: definition-I}
I=\int_{0}^1 \left(\frac{r^2}{1+r^2t}\right)^{\frac{k}{2}} t^{\frac{k-1}{2}}(1-t)^{\frac{k+d-3}{2}} \dd t.
\end{equation}
Now, our task is to estimate how $I$ depends on $r,k$, and $d$. In order to achieve this, we consider two cases: (1) $r^2\geq \frac{2k+d-4}{k-1}$; (2) $r^2\leq \frac{2k+d-4}{k-1}$, separately.

\begin{lemma}\label{lemma: h}
Let $h_{\alpha, \beta}(t) = t^{\alpha}(1-t)^{\beta}$.  Then, $h_{\alpha,\beta}$ is increasing in $[0,\frac{\alpha}{\alpha+\beta}]$ and decreasing in $[\frac{\alpha}{\alpha+\beta},1]$.
\end{lemma}
\begin{proof}
A simple calculation gives us 
$$
h'_{\alpha,\beta}(t)=\alpha t^{\alpha-1}(1-t)^{\beta} -\beta t^\alpha (1-t)^{\beta-1} = \Big(\alpha -(\alpha+\beta) t\Big) t^{\alpha-1} (1-t)^{\beta-1}.
$$
Hence, $h'_{\alpha,\beta}(t)\geq 0$ for $t\in [0,\frac{\alpha}{\alpha+\beta}]$, and $h'_{\alpha,\beta}(t)\leq 0$ for $t\in [\frac{\alpha}{\alpha+\beta},1]$
\end{proof}

\begin{proposition}\label{pro: arctan-case-1}
Assume $r^2\geq \frac{2k+d-4}{k-1}$. Then, $I\lesssim r e^{-\frac{k+d}{2r^2}}$.
\end{proposition}
\begin{proof}
Consider the decomposition
\begin{align*}
\notag I&=  \left(\int_{0}^{\ia} + \int_{\ia}^{1}\right) \left(\frac{r^2}{1+r^2t}\right)^{k/2} t^{\frac{k-1}{2}}(1-t)^{\frac{k+d-3}{2}} \dd t \\
&=: I_1 + I_2. 
\end{align*}
Next, we estimate $I_1,I_2$, separately. 

By Lemma \ref{lemma: h},  $h_{\frac{k-1}{2}, \frac{k+d-3}{2}}(\cdot)$ is increasing for $t\leq \bar{t}=\frac{k-1}{2k+d-4}$. The assumption on $r$ ensures that $\frac{1}{r^2}\leq \bar{t}$.
Therefore, for $I_1$,
\begin{align}\label{eqn: I-I_1}
I_1 &\leq  \int_0^{\ia} \frac{r^k}{(1+r^2t)^{k/2}}\, \left(\ia\right)^{\frac{k-1}{2}} \left(1-\frac{1}{r^2}\right)^{\frac{k+d-3}{2}} \dd t\lesssim \frac{1}{kr}  \left(1-\ia\right)^{\frac{k+d-3}{2}}\lesssim \frac{1}{k}\left(1-\ia\right)^{\frac{k+d}{2}},
\end{align}
where the last inequality uses the fact that $r^2\geq 2$.

For $I_2$, 
\begin{align}\label{eqn: I-I_2}
\notag I_2 &= \int_{\ia}^{1}  \left(\frac{r^2t}{1+r^2t}\right)^{\frac{k}{2}} t^{-\frac{1}{2}} (1-t)^{\frac{k+d-3}{2}} \dd t\leq \int_{\ia}^{1} t^{-\frac{1}{2}} (1-t)^{\frac{k+d-3}{2}} \dd t \\
&\stackrel{(i)}{\leq}  r\left(1-\ia\right)^{\frac{k+d-3}{2}} \lesssim r\left(1-\ia\right)^{\frac{k+d}{2}},
\end{align}
where $(i)$ is due to  that $h_{-\frac{1}{2}, \frac{k+d-3}{2}}(\cdot)$ is decreasing in $[0,1]$.

Combining \eqref{eqn: I-I_1} and \eqref{eqn: I-I_2}, we have 
\begin{equation*}
I\lesssim  r\left(1-\ia\right)^{\frac{k+d}{2}}\leq r e^{-\frac{k+d}{2r^2}}.
\end{equation*}
\end{proof}

\begin{lemma}\label{lemma: A-func}
For $\alpha,\beta,\gamma>0$, let
$
    H_{\gamma,\alpha,\beta}(t) := \left(\frac{\gamma t}{1+\gamma t}\right)^\alpha (1-t)^\beta
$
with $t\in [0,1]$.
Assume $\beta\geq \alpha$ and $\gamma \alpha/ \beta\leq 2$.
Then, 
\[
    H_{\gamma,\alpha,\beta}(t)\leq \left(\frac{\gamma \alpha}{\gamma \alpha+\beta}  \right)^\alpha\left(1-\frac{\alpha}{2\beta}\right)^\beta.
\]
\end{lemma}
\begin{proof}
Taking the derivative gives us 
\begin{align*}
\notag H'_{\gamma,\alpha,\beta}(t) &= \alpha \left(\frac{\gamma t}{1+\gamma t}\right)^{\alpha-1} \frac{\gamma}{(1+\gamma t)^2} (1-t)^\beta- \beta \left(\frac{\gamma t}{1+\gamma t}\right)^\alpha (1-t)^{\beta-1}\\
&= \Big(\alpha - \beta  t(1+\gamma t)\Big)\frac{\gamma }{(1+\gamma t)^2}\left(\frac{\gamma t}{1+\gamma t}\right)^{\alpha-1} (1-t)^{\beta-1}
\end{align*}
Let $\delta=\alpha/\beta$. Then, the maximal value of $H_{\gamma,\alpha,\beta}(\cdot)$ is reached at 
\[
    \bar{t} = \frac{1}{2\gamma }\left(\sqrt{4\gamma \delta+1} - 1\right).
\]
Notice that $1+t/4\leq \sqrt{1+t}\leq 1+t/2$ for $t\in [0,8]$. By the assumption, $0\leq 4\gamma \delta\leq 8$, which implies that
$
\frac{\delta}{2}\leq \bar{t}\leq \delta.
$
Hence, 
\[
H_{\gamma,\alpha,\beta}(t)\leq H_{\gamma,\alpha,\beta}(\bar{t}) = \left(\frac{\gamma \bar{t}}{1+\gamma \bar{t}}\right)^\alpha (1-\bar{t})^\beta \leq \left(\frac{\gamma \delta}{1+r\delta}\right) \left(1-\frac{\delta}{2}\right)^\beta = \left(\frac{\gamma\alpha}{\gamma \alpha+\beta}  \right)^\alpha\left(1-\frac{\alpha}{2\beta}\right)^\beta.
\]
\end{proof}

\begin{lemma}\label{pro: arctan-case-2}
Assume that $r^2\leq \frac{2k+d-4}{k-1}$ and $k,d\geq C$ for an absolute constant $C$. Then, we have 
\[
I\lesssim d^{\frac{1}{2}}r\left(\frac{r^2k}{(r^2+1)k+d}\right)^{\frac{k}{2}} e^{-k/4}.
\]
\end{lemma}

Recalling \eqref{eqn: definition-I},
\begin{align*}
I &\leq r \int_{0}^1 \left(\frac{r^2t}{1+r^2t}\right)^{\frac{k-1}{2}} (1-t)^{\frac{k+d-3}{2}} \dd t =  r \int_0^1 H_{r^2,\frac{k-1}{2},\frac{k+d-3}{2}}(t)\dd t.
\end{align*}
By the assumption, 
\[
    r^2 \frac{(k-1)/2}{(k+d-3)/2}\leq \frac{2k+d-4}{k+d-3} = 2 -\frac{d-2}{k+d-3}\leq 2,
\]
which means that the condition in Lemma \ref{lemma: A-func} is satisfied.  Therefore, we have 
\begin{align}
\notag I&\leq r \left(\frac{r^2 (k-1)}{r^2(k-1)+k+d-3}\right)^{\frac{k-1}{2}} \left(1-\frac{k-1}{2(k+d-3)}\right)^{\frac{k+d-3}{2}}\\
\notag &= r \left(\frac{r^2(k-1)}{(r^2+1)(k-1)+d-2}\right)^{\frac{k-1}{2}}\left(1-\frac{k-1}{2(k+d-3)}\right)^{\frac{k+d-3}{2}}\\
&\lesssim d^{\frac{1}{2}}r\left(\frac{r^2k}{(r^2+1)k+d}\right)^{\frac{k}{2}}\left(1-\frac{k}{2(k+d)}\right)^{\frac{k+d}{2}},
\end{align}
where the last inequality follows from that $d,k\geq C$ for some absolute constant $C$. Using the fact that $(1-1/x)^x\leq e^{-1}$ for any $x\geq 1$, we complete the proof.

\paragraph*{Proof of Proposition \ref{pro: 2lnn-smooth-ub-2}}
Following Theorem \ref{proposition: 2lnn-complete-characterization}, what remains is to show that when $\sigma(x) = \arctan(\gamma x+ b)$ such that $|\gamma| + |b| \le r$, we have 
\[
\Lambda(m)\lesssim d^2r^4 m^{-\frac{1}{\max(2,r^2)}}.
\]
According to Lemma \ref{lemma: eigen} and Proposition \ref{pro: arctan-analytic-expression}, $\mu_k=Q_{d,k}^2 I^2$ if $k$ is odd, other wise $\mu_k=0$. Assume that $k$ is odd. 
Let 
\begin{equation}\label{eqn: m-est}
    m_k=\frac{\Gamma(k+d)}{\Gamma(d)\Gamma(k+1)}\sim \frac{(k+d)^{k+d-\frac{1}{2}}}{k^{k+\frac{1}{2}}d^{d-\frac{1}{2}}}\lesssim \frac{(k+d)^{k+d}}{k^kd^d}.
\end{equation}
Our task is to determine the smallest $\beta$ such that 
\[
    m_k^{1+1/\beta} \mu_k \le m_k^{1+1/\beta} Q_{d,k}^2 I^2=\left(m_k Q_{d,k}^2\right) \left(m_k^{1/\beta} I^2\right)\leq C_{d,r},
\]
where $C_{d,r}$ depends on $d,r$ polynomially. 

First, 
\begin{align}\label{eqn: q_dk}
\notag m_k Q_{d,k}^2&=\frac{\Gamma(k+d)}{\Gamma(d)\Gamma(k+1)} \left(\frac{\Gamma(k)\Gamma(\frac{d}{2})}{2^k\Gamma(\frac{k+1}{2})\Gamma(\frac{k+d-1}{2})}\right)^2 \\
&\sim \frac{(k+d)^{k+d-\frac{1}{2}}}{k^{\frac{1}{2}+k}d^{d-\frac{1}{2}}} \cdot \frac{k^{k-1}d^{d-1}}{(k+d)^{k+d-2}} \lesssim d.
\end{align}
Next, we turn to estimate $m^{1/\beta}_k I^2$. 

\textbf{(1)} We first consider  the case where $r^2\leq \frac{2k+d-4}{k-1}$. By Proposition \ref{pro: arctan-case-2} and \eqref{eqn: m-est}, we have 
\begin{align}
\notag m_k^{1/\beta} I^2 &\lesssim d  r^2\left(\frac{(k+d)^{k+d}}{k^kd^d}\right)^{1/\beta} \left(\frac{r^2k}{(r^2+1)k+d}\right)^{k}e^{-k/2}\\
\notag &\lesssim d r^2\left(\frac{(k+d)(r^2k)^\beta}{k ((r^2+1)k+d)^\beta}\right)^{k/\beta}\left(\left(\frac{k+d}{d}\right)^{d/\beta} e^{-k/2}\right)\\
&=: dr^2 A^{k/\beta} B.
\end{align}
Next, we estimate $A$ and $B$, separately. To simplify the notation, we write $k=sd, a=r^2$. For $A$, if $\beta\geq a$, 
\begin{align} \label{eqn: A-part}
A = \frac{(k+d)(r^2k)^\beta}{k ((r^2+1)k+d)^\beta} = \frac{(s+1)(as)^\beta}{s((a+1)s+1)^\beta} \leq \frac{a^\beta s^{\beta+1} + a^\beta s^\beta}{(a+1)^\beta s^{\beta+1}+\beta (a+1)^{\beta-1}s^\beta}\leq 1.
\end{align}
For $B$, if $\beta\geq 2$,
\begin{align}\label{eqn: B-part}
    B&= \left(1+\frac{k}{d}\right)^{d/\beta}e^{-k/2}\leq e^{k/\beta}\cdot e^{-k/2}\leq 1.
\end{align}
Combining \eqref{eqn: A-part} and \eqref{eqn: B-part}, 
\[
    m_k^{1/r^2} I^2\leq dr^2.
\]

\textbf{(2)} Now, we turn to the case where $r^2\geq \frac{2k+d-4}{k-1}$. By Proposition \eqref{pro: arctan-case-1} and \eqref{eqn: m-est}, 
\begin{align*}
\notag m_k^{1/\beta} I^2 &\lesssim \left(\frac{(k+d)^{k+d}}{k^kd^d}\right)^{1/\beta} r^2e^{-(k+d)/r^2}\\
\notag & \leq r^2\left(1+\frac{d}{k}\right)^{k/\beta} \left(1+\frac{k}{d}\right)^{d/\beta} e^{-(k+d)/r^2}\\
&\leq r^2 e^{d/\beta} e^{k/\beta} e^{-(k+d)/r^2}.
\end{align*}
Taking $\beta=r^2$ gives rise to $m_k^{1/r^2} I^2 \lesssim r^2$.

Combining two cases, we arrive at 
\begin{align*}
m^{1+\frac{1}{\max(2,r^2)}}\mu_k \lesssim d^2r^2.
\end{align*}
Therefore, 
\[
    \Lambda(m) = \sum_{j=m+1}^\infty \bar{\lambda}_j\lesssim d^2r^2\sum_{j=m+1}^\infty j^{-1-\frac{1}{\max(2,r^2)}}\lesssim d^2r^4 m^{-\frac{1}{\max(2,r^2)}}.
\]

\subsection{The missing proof of Theorem \ref{thm: 2lnn-arctan-lower-bound}}
\label{sec: missing-proof-thm3}
\begin{proof}
Here, we provide the proof for the ReLU-like case, which is similar to the sigmoid-like case. 
For any $v\in\SS^{d-1}$, we have 
\begin{align*}
    \|\sigma(rv^T\cdot)-&r\relu(v^T\cdot)\|_{\tau_{d-1}}^2 = \frac{1}{B(\frac{1}{2}, \frac{d-1}{2})}\int_{-1}^1  (\sigma(rt)-r\relu(t))^2 (1-t^2)^{\frac{d-3}{2}}\dd t\\
    &\leq \frac{2}{B(\frac{1}{2}, \frac{d-1}{2})}\left(\int_0^{\delta}(1-t^2)^{\frac{d-3}{2}}\dd t + \int_{\delta}^1  \frac{1}{(1+rt)^{2\beta}}(1-t^2)^{\frac{d-3}{2}}\dd t\right)\\
    &\lesssim d^{\frac{1}{2}}\delta +\frac{1}{(1+r\delta)^{2\beta}} \leq d^{\frac{1}{2}}\delta +\frac{1}{(r\delta)^{2\beta}}.
\end{align*}
Following the proof of Theorem \ref{thm: 2lnn-arctan-lower-bound}, we have 
\[ 
\inf_{\delta\in [0,1]} \left(d^{\frac{1}{2}}\delta + \frac{1}{(r\delta)^{2\beta}}\right) \leq  C(\beta) \left(\frac{d}{r^2}\right)^{\frac{\beta}{2(1+2\beta)}}.
\]
By the triangular inequality,
\begin{align*}
\notag \|\sigma(rv^T\cdot)-\sum_{j=1}^m c_j \phi_j\|^2_{\tau_{d-1}}&\gtrsim \|r\relu(v^T\cdot)-\sum_{j=1}^m c_j \phi_j\|^2_{\tau_{d-1}} - \|r\relu(v^T\cdot)-\sigma(rv^T\cdot)\|^2_{\tau_{d-1}}\\
&\gtrsim \|r\relu(v^T\cdot)-\sum_{j=1}^m c_j \phi_j\|^2_{\tau_{d-1}} - C(\beta) \left(\frac{d}{r^2}\right)^{\frac{\beta}{2(1+2\beta)}}.
\end{align*}
Using Theorem \ref{thm: gen-1}, there exists a $C_3>0$ such that 
\begin{align*}
\notag \EE_{v\sim\tau_{d-1}}\inf_{c_1,\dots,c_m}\|\sigma(rv^T\cdot)-\sum_{j=1}^m c_j \phi_j\|^2_{\tau_{d-1}}
&\gtrsim \frac{r^2}{d^{C_3}m^{3/(d-1)}} - C(\beta) \left(\frac{d}{r^2}\right)^{\frac{\beta}{2(1+2\beta)}}.
\end{align*}
Obviously, there exists $C'(\beta)>0$ such that when  $m\leq 2^d$ and $r\geq C'(\beta)d^{\max(1/2,C_3/2)}$, we  have 
$$
\EE_{v\sim\tau_{d-1}}\inf_{c_1,\dots,c_m}\|\sigma(rv^T\cdot)-\sum_{j=1}^m c_j \phi_j\|^2_{\tau_{d-1}} \gtrsim 1.
$$ 
\end{proof}

\subsection{Proof of Theorem \ref{thm: low-bound-actan}}
\label{sec: proof-lower-bound-actan}
\begin{proof}
Following the same approach with Proposition \ref{pro: arctan-analytic-expression}, we know that
\begin{equation*}
    |\eta_k| = \begin{cases}
    Q_{d,k}\int_{0}^1 \left(\frac{d^{2\alpha}}{1+d^{2\alpha}t}\right)^{\frac{k}{2}} t^{\frac{k-1}{2}}(1-t)^{\frac{k+d-3}{2}} \dd t &\text{ when }k\text{ is odd} \\
    0 &\text{ when }k\text{ is even}.
    \end{cases}
\end{equation*}
Therefore, when $\frac{\Gamma(2k+d-3)}{\Gamma(d)\Gamma(2k-2)} +1 \le m \le \frac{\Gamma(2k+d-1)}{\Gamma(d)\Gamma(2k)}$ or $m =k=1$, we have
\begin{equation*}
    \mu_m^{\frac{1}{2}} = Q_{d,2k-1}\int_{0}^{1}\left(\frac{d^{2\alpha}}{1+d^{2\alpha}t}\right)^{\frac{2k-1}{2}} t^{k-1}(1-t)^{\frac{2k+d-4}{2}}\dd t.
\end{equation*}

Noticing that $\Lambda(m)$ is increasing with respect to $\alpha$,
we only consider the case that $\alpha \in (\frac{1}{2},1]$. 

First, noticing that
\begin{align*}
    \int_{0}^{1}\left(\frac{d^{2\alpha}}{1+d^{2\alpha}t}\right)^{\frac{2k-1}{2}} t^{k-1}(1-t)^{\frac{2k+d-4}{2}}\dd t &\ge \int_{\frac{1}{2d}}^{\frac{1}{d}}\left(\frac{d^{2\alpha}}{1+d^{2\alpha}t}\right)^{\frac{2k-1}{2}} t^{k-1}(1-t)^{\frac{2k+d-4}{2}}\dd t \\
    &\gtrsim d^{-\frac{3}{2}}(\frac{d^{2\alpha-1}}{2+d^{2\alpha-1}})^{\frac{2k-1}{2}}(1-\frac{1}{d})^{\frac{2k+d-4}{2}} \\
    &\gtrsim d^{-\frac{3}{2}}(\frac{d^{2\alpha-1}}{2+d^{2\alpha-1}})^{2k}.
\end{align*}
Also,
\begin{align*}
    N(d,2k-1)Q^2_{d,2k-1} &= \frac{4k+d-4}{2k-1}\frac{\Gamma(2k+d-3)}{\Gamma(d-1)\Gamma(2k-1)}Q^2_{d,2k-1}\\
    &\gtrsim \frac{(2k+d)^{2k+d-\frac{7}{2}}}{(2k)^{2k-\frac{3}{2}}d^{d-\frac{3}{2}}}\frac{(2k)^{2k-2}d^{d-1}}{(2k+d)^{2k+d-3}}\\
    &=\left(\frac{2kd}{(2k+d)}\right)^{\frac{1}{2}}\gtrsim 1.
\end{align*}
Therefore, for any $k \ge 1$,
\begin{align*}
    \Lambda\left(\frac{\Gamma(2k+d-3)}{\Gamma(d)\Gamma(2k-2)}\right) &=\sum_{j = k}^{+\infty}N(d,2j-1)Q_{d,2j-1}^2 \left[\int_{0}^{1}\left(\frac{d^{2\alpha}}{1+d^{2\alpha}t}\right)^{\frac{2k-1}{2}} t^{k-1}(1-t)^{\frac{2k+d-4}{2}}\dd t\right]^2 \\
    &\gtrsim d^{-3}\sum_{j=k}^{+\infty}(\frac{d^{2\alpha-1}}{2+d^{2\alpha-1}})^{4j}\gtrsim d^{-3} (\frac{d^{2\alpha-1}}{2+d^{2\alpha-1}})^{4k}.
\end{align*}
Therefore, when $k \sim \frac{1}{2}d^{2\alpha-1}$, we have
\begin{equation}\label{eqn: thm_3_1}
     \Lambda\left(\frac{\Gamma(2k+d-3)}{\Gamma(d)\Gamma(2k-2)}\right) \gtrsim d^{-3}.
\end{equation}

On the other hand, when $k \sim \frac{1}{2}d^{2\alpha-1}$, we have
\begin{equation}\label{eqn: thm_3_2}
    m_{d,\alpha}:=\frac{\Gamma(2k+d-3)}{\Gamma(d)\Gamma(2k-2)}\sim \frac{(2k+d)^{2k+d-\frac{7}{2}}}{(2k)^{2k-\frac{5}{2}}d^{d-\frac{1}{2}}}
\end{equation}
Therefore, there exist $C_1, C_2(\alpha)>0$ such that 
\begin{equation}
 m_{d,\alpha}\gtrsim d^{(2-2\alpha)d^{2\alpha-1} - \frac{11}{2} +5\alpha} = d^{-\frac{11}{2}+5\alpha}2^{C(\alpha)d^{2\alpha-1}\ln(d)}\geq \frac{1}{d^{C_1}}2^{C_2(\alpha)d^{2\alpha-1}}.
\end{equation}

% If there exists a polynomial $C(d)$ such that $\Lambda_\pi(m)\le C(d) m^{-\beta}$, then
% combining inequality \eqref{eqn: thm_3_1} and \eqref{eqn: thm_3_2}, we obtain that
% \begin{equation*}
%     d^{\beta[(2-2\alpha)d^{2\alpha-1} - \frac{11}{2} +5\alpha] + 3} \le C(d),
% \end{equation*}
% which is a contradiction.
\end{proof}

\vskip 0.2in
\bibliography{ref}

\end{document}